\pdfoutput=1
\documentclass{article}

\usepackage[final,nonatbib]{neurips_2021}

\usepackage[utf8]{inputenc} 
\usepackage[T1]{fontenc}    
\usepackage{hyperref}       
\usepackage{url}            
\usepackage{booktabs}       
\usepackage{amsfonts}       
\usepackage{nicefrac}       
\usepackage{microtype}      
\usepackage{xcolor}         
\usepackage{amsmath}
\usepackage{graphicx}
\usepackage{subfigure}
\usepackage{amsthm,amsmath,amssymb}
\usepackage{mathrsfs}
\usepackage{booktabs}
\usepackage{tablefootnote}
\usepackage{threeparttable}
\usepackage{bm}
\usepackage[noend]{algpseudocode}
\usepackage{algorithmicx,algorithm}
\usepackage{caption}
\usepackage{multirow}

\title{Training Feedback Spiking Neural Networks by Implicit Differentiation on the Equilibrium State}

\author{%
  Mingqing Xiao$^1$, Qingyan Meng$^{2,3}$, Zongpeng Zhang$^{1,4}$, Yisen Wang$^1$, Zhouchen Lin$^{1,5\thanks{Corresponding author.}}$\\
  $^1$Key Laboratory of Machine Perception (MOE), School of AI, Peking University\\
  $^2$The Chinese University of Hong Kong, Shenzhen\\
  $^3$Shenzhen Research Institute of Big Data\\
  $^4$Center for Data Science, Academy for Advanced Interdisciplinary Studies, Peking University\\
  $^5$Pazhou Lab, Guangzhou 510330, China\\
  \texttt{\{mingqing\_xiao, yisen.wang, zlin\}@pku.edu.cn, qingyanmeng@link.cuhk.edu.cn,}\\ \texttt{zongpeng.zhang98@gmail.com} \\
}

\begin{document}
\maketitle

\begin{abstract}
  Spiking neural networks (SNNs) are brain-inspired models that enable energy-efficient implementation on neuromorphic hardware. However, the supervised training of SNNs remains a hard problem due to the discontinuity of the spiking neuron model. Most existing methods imitate the backpropagation framework and feedforward architectures for artificial neural networks, and use surrogate derivatives or compute gradients with respect to the spiking time to deal with the problem. 
  These approaches either accumulate approximation errors or only propagate information limitedly through existing spikes, and usually require information propagation along time steps with large memory costs and biological implausibility. In this work, we consider feedback spiking neural networks, which are more brain-like, and propose a novel training method that does not rely on the exact reverse of the forward computation. First, we show that the average firing rates of SNNs with feedback connections would gradually evolve to an equilibrium state along time, which follows a fixed-point equation. Then by viewing the forward computation of feedback SNNs as a black-box solver for this equation, and leveraging the implicit differentiation on the equation, we can compute the gradient for parameters without considering the exact forward procedure. In this way, the forward and backward procedures are decoupled and therefore the problem of non-differentiable spiking functions is avoided. We also briefly discuss the biological plausibility of implicit differentiation, which only requires computing another equilibrium. Extensive experiments on MNIST, Fashion-MNIST, N-MNIST, CIFAR-10, and CIFAR-100 demonstrate the superior performance of our method for feedback models with fewer neurons and parameters in a small number of time steps. Our code is available at \url{https://github.com/pkuxmq/IDE-FSNN}.
\end{abstract}

\vspace{-2mm}
\section{Introduction}
\vspace{-2mm}

Spiking neural networks (SNNs) have gained increasing attention recently due to their inherent energy-efficient computation~\cite{lee2016training,shrestha2018slayer,wu2018spatio,roy2019towards,deng2021optimal}. Inspired by the neurons in the human brain, biologically plausible SNNs transmit spikes between neurons, enabling event-based computation which can be carried out on neuromorphic chips with less energy consumption~\cite{akopyan2015truenorth,davies2018loihi,pei2019towards,roy2019towards}. Meanwhile, 
SNNs are computationally more powerful than artificial neural networks (ANNs) theoretically and are therefore regarded as the third generation of neural network models~\cite{maass1997networks}.

Despite the advantages, directly supervised training of SNNs remains a hard problem, which hampers the real applications of SNNs compared with popular ANNs. The main obstacle lies in the complex spiking neuron model. While backpropagation~\cite{rumelhart1986learning} works well for ANNs, it suffers from the discontinuity of spike generation which is non-differentiable in SNN training. 
Most recent successful SNN training methods still imitate the backpropagation through time (BPTT)~\cite{werbos1990backpropagation} framework by error propagation through the computational graph unfolded along time steps, and they deal with the spiking function by applying surrogate derivatives to approximate the gradients~\cite{wu2018spatio,bellec2018long,jin2018hybrid,shrestha2018slayer,wu2019direct,neftci2019surrogate,kim2020unifying,zheng2020going}, or by computing the gradients with respect to the spiking time only on the spiking neurons~\cite{bohte2002error,zhang2020temporal,kim2020unifying}. However, these methods either accumulate approximation error along time steps, or suffer from the ``dead neuron'' problem~\cite{shrestha2018slayer}, i.e. learning would not occur when no neuron spikes. At the same time, BPTT requires memorizing intermediate variables at all time steps and backpropagating along them, which is memory-costing and biologically implausible. So it is necessary to consider training methods other than backpropagation along computational graphs that fit SNNs better.

On the other hand, most recent SNN models simply imitate the feedforward architectures of ANNs~\cite{wu2018spatio,shrestha2018slayer,wu2019direct,zhang2020temporal,zheng2020going}, which ignores the ubiquitous feedback connections in the human brain. 
Feedback (recurrent) circuits are critical to human's vision system for object recognition~\cite{kar2019evidence}. Meanwhile, \cite{kubilius2019brain} shows that shallow ANNs with recurrence achieve higher functional fidelity of human brains and similarly high performance on large-scale vision recognition tasks, compared with deep ANNs. So incorporating feedback connections enables neural networks to be shallower, more efficient, and more brain-like.
As for SNNs, feedback was popular in early models like Liquid State Machine~\cite{maass2002real}, which leverages a recurrent reservoir layer with weights fixed or trained by unsupervised methods. And compared with the uneconomical cost for ANNs to incorporate feedback connections by unfolding along time, SNNs naturally compute with multiple time steps, which inherently supports feedback connections. 
Most recent SNN models imitate feedforward architectures because they were once lacking effective training methods and thus they borrow everything from successful ANNs. We focus on another direction, i.e. feedback SNN, which is a natural choice for visual tasks as well.

In this work, we consider the training of feedback spiking neural networks (FSNN), and propose a novel method based on the Implicit Differentiation on the Equilibrium state (IDE). Inspired by recent advances in implicit models~\cite{bai2019deep,bai2020multiscale}, which treat weight-tied ANNs as solving a fixed-point equilibrium equation and propose alternative implicit models defined by the equation, we derive that when the average inputs converge to an equilibrium, the average firing rates of FSNNs would gradually evolve to an equilibrium state along time, which follows a fixed-point equation as well. Then we view the forward computation of FSNN as a black-box solver for the fixed-point equation, and borrow the idea of implicit differentiation from implicit models~\cite{bai2019deep,bai2020multiscale} to calculate the gradients, which only relies on the equation rather than the exact forward procedure. In this way, gradient calculation is agnostic to the spiking function in SNN, thus avoiding the common difficulties in SNN training. 
While implicit differentiation may seem too abstract to be computed in the brain, we briefly discuss the biological plausibility and show that it only requires computing another equilibrium along the inverse connections of neurons. Besides, we incorporate the multi-layer structure into the feedback model for better representation ability. Our contributions include: 
\vspace{-1mm}
\begin{enumerate}
    \item We are the first to theoretically derive the equilibrium states with a fixed-point equation for the average firing rates of FSNNs with the (leaky) integrate and fire model under both continuous and discrete views. According to this, the forward computation of FSNNs can be interpreted as solving a fixed-point equation.
\vspace{-0.25mm}
    \item We propose a novel training method for FSNNs based on the implicit differentiation on the equilibrium state, which is decoupled from the forward computational graph and avoids SNN training problems, e.g. non-differentiability and large memory costs. We also discuss the biological plausibility and demonstrate the connection to the Hebbian learning rule.
\vspace{-0.25mm}
    \item We conduct extensive experiments on MNIST, Fashion-MNIST, N-MNIST, CIFAR-10, and CIFAR-100, which demonstrate the superior results of our methods with fewer neurons and parameters in a small number of time steps for both static images and neuromorphic inputs. Especially, our directly trained model can outperform the state-of-the-art SNN performance on the complex CIFAR-100 dataset with only 30 time steps.
\end{enumerate}
\vspace{-2.5mm}

\section{Related Work}
\vspace{-2mm}

\textbf{Training Methods for Spiking Neural Networks.}\quad
Early works apply biologically inspired method, spike-time dependent plasticity (STDP)~\cite{diehl2015unsupervised}, to formulate a bottom-up unsupervised learning rule, or choose reward-modulated STDP~\cite{legenstein2008learning} with limited supervision. 
Since the rise of successful ANNs, error backpropagation and gradient descent have inspired many methods. One direction is to convert a trained ANN to SNN~\cite{hunsberger2015spiking,rueckauer2017conversion,sengupta2019going,rathi2019enabling,deng2021optimal,yan2021near}. However, they suffer from conversion errors and extremely large simulation time steps. The other methods are to directly calculate the gradient and train SNNs. These methods either compute the gradient with respect to spiking times~\cite{bohte2002error,zhang2020temporal}, or leverage a surrogate derivative for discontinuous spiking functions~\cite{lee2016training,wu2018spatio,bellec2018long,jin2018hybrid,shrestha2018slayer,wu2019direct,neftci2019surrogate,zheng2020going}, or combine them~\cite{kim2020unifying}.  
However, they suffer from the ``dead neuron'' problem~\cite{shrestha2018slayer} or accumulated approximation error, and typically require backpropagation along the computational graph to be unfolded by many time steps, which is memory-consuming and biologically implausible. As for SNN with feedback connection, ~\cite{zhang2019spike} proposes the ST-RSBP method, which backpropagates errors at the spike-train level. They simply approximate the activation function of a neuron as a linear operation, and require long time steps for satisfactory results. In this work, we propose a new training method that does not rely on the exact reverse of the forward computation, which avoids problems of non-differentiability and large memory costs, and only requires short time steps for superior performance. There are also works trying methods other than BP along computational graphs to train SNNs, such as equilibrium propagation~\cite{o2019training}. However, \cite{o2019training} defines a complex computation form rather than common SNN models, and can hardly achieve satisfactory results on the simple MNIST dataset. Instead, our work is based on SNN models applicable on neuromorphic hardware and demonstrates promising results on more complex datasets.

\textbf{Equilibrium of Neural Networks.}\quad
The study on the equilibrium of neural networks originates from energy-based models, e.g. Hopfield Network~\cite{hopfield1982neural,hopfield1984neurons}. They view the dynamics or iterative procedures of feedback (recurrent) neural networks as minimizing an energy function, which will converge to a minimum of the energy. 
Based on the energy, several training methods are proposed, including recurrent backpropagation~\cite{almeida1987learning,pineda1987generalization} and more recent equilibrium propagation (EP)~\cite{scellier2017equilibrium}. They heavily rely on the energy function and can hardly achieve competitive results as deep neural networks do. Deep equilibrium models~\cite{bai2019deep,bai2020multiscale}, on the other hand, are recently proposed models which treat weight-tied deep ANNs as solving a fixed-point equilibrium point, and propose alternative implicit models defined by the fixed-point equations rather than energy functions. They express the entire deep network as an equilibrium computation and solve and train it by root-finding methods and implicit differentiation, respectively, which achieves superior results. 
Most of these works are based on ANNs, except that \cite{o2019training} generalizes the EP method to neurons with binary communications. They define a complex neuron computation form and follow the methodology of energy-based EP method to approximate the gradients. Several recent works also study the equilibrium of SNNs~\cite{li2020minimax,mancoo2020understanding}. They consider equilibrium from the perspective of solving a constrained optimization problem, but either do not propose to guide the training based on it or are limited in single-layer structure and simple problems. 
Differently, we are the first to derive the equilibrium state with a fixed-point equation for average firing rates of FSNNs with common SNN models, and propose to train SNNs by exact gradients through implicit differentiation, which is also scalable to multi-layer structure and deep learning problems.

\section{Preliminaries}

\subsection{Spiking Neural Network Models}

Spiking neurons, different from real-valued ANNs, communicate with each other by binary spike trains along time. Each neuron maintains a membrane potential, which integrates input spike trains, and the neuron would generate an output spike once the membrane potential exceeds a threshold. The commonly used integrate and fire (IF) model and leaky integrate and fire (LIF) model describe the dynamics of the membrane potential as:
\begin{equation}
    \begin{aligned}
        \textbf{IF:}\quad &\frac{du}{dt} = R\cdot I(t),\quad &u < V_{th},\\
        \textbf{LIF:}\quad &\tau_m\frac{du}{dt} = -(u-u_{rest}) + R\cdot I(t),\quad &u < V_{th},
    \end{aligned}
\end{equation}
where $u$ is the membrane potential, $I$ is the input current, $V_{th}$ is the spiking threshold, and $R$ and $\tau_m$ are resistance and time constant, respectively. 
Once $u$ reaches $V_{th}$ at time $t^f$, a spike is generated and $u$ is reset to the resting potential $u=u_{rest}$, which is usually taken as $0$. The spike train is expressed by the Dirac delta function: $s(t) = \sum_{t^f}\delta(t-t^f)$. 
We consider the simple current model $I_i(t)=\sum_j w_{ij}s_j(t) +b$, where 
$w_{ij}$ is the weight from neuron $j$ to neuron $i$, which is the same as ANN. After discretization, the general computation form for the IF and LIF model is described as:
\begin{equation}
    \left\{
    \begin{aligned}
        &u_i\left[t + 0.5\right] = \lambda u_i[t] + \sum_j w_{ij}s_j[t] + b,\\
        &s_i[t + 1] = H(u_i\left[t+0.5\right] - V_{th}),\\
        &u_i[t + 1] = u_i\left[t + 0.5\right] - V_{th}s_i[t + 1],
    \end{aligned}
    \right.
    \label{eq.discrete}
\end{equation}
where $H(x)$ is the Heaviside step function, i.e. the non-differentiable spiking function,  $s_i[t]$ is the binary spike train of neuron $i$, and $\lambda$ is 1 for the IF model while $\lambda<1$ is a leaky term related to the constant $\tau_m$ and discretization time interval for the LIF model. The constant $R$, $\tau_m$, and time step size are absorbed into the weights $w_{ij}$ and bias $b$. 
We use subtraction as the reset operation.

\subsection{Implicit Differentiation on the Fixed-Point Equation}\label{sec:id}

We consider a fixed-point equation $\mathbf{a}=f_{\theta}(\mathbf{a})$ parameterized by $\theta$. Let $\mathcal{L}(\mathbf{a}^*)$ denote the objective function with respect to the equilibrium state $\mathbf{a}^*$, and let $g_{\theta}(\mathbf{a})=f_{\theta}(\mathbf{a})-\mathbf{a}$. The implicit differentiation on the equation satisfies $\left(I-\frac{\partial f_{\theta}(\mathbf{a}^*)}{\partial \mathbf{a}^*}\right)\frac{\mathrm{d}\mathbf{a}^*}{\mathrm{d} \theta}=\frac{\partial f_{\theta}(\mathbf{a}^*)}{\partial \theta}$~\cite{bai2019deep}.
Therefore, the differentiation of $\mathcal{L}(\mathbf{a}^*)$ with respect to parameters can be calculated based on implicit differentiation:
\begin{equation}
    \frac{\partial \mathcal{L}(\mathbf{a}^*)}{\partial \theta} = -\frac{\partial \mathcal{L}(\mathbf{a}^*)}{\partial \mathbf{a}^*} \left(J_{g_{\theta}}^{-1}\vert_{\mathbf{a}^*}\right) \frac{\partial f_{\theta}(\mathbf{a}^*)}{\partial \theta},
\end{equation}
where $J_{g_{\theta}}^{-1}\vert_{\mathbf{a}^*}$ is the inverse Jacobian of $g_{\theta}$ evaluated at $\mathbf{a}^*$. For the proof please refer to \cite{bai2019deep}. To solve the inverse Jacobian, we follow \cite{bai2019deep,bai2020multiscale} and solve an alternative linear system $\left(J_{g_{\theta}}^T\vert_{\mathbf{a}^*}\right)\mathbf{x}+\left(\frac{\partial \mathcal{L}(\mathbf{a}^*)}{\partial \mathbf{a}^*}\right)^T=0$. We can leverage Broyden’s method~\cite{bai2019deep,bai2020multiscale}, which is a second-order quasi-Newton approach; or we can alternatively use a fixed-point update scheme 
$\mathbf{x}=(J_{f_{\theta}}^T\vert_{\textbf{a}^*})\mathbf{x}+\left(\frac{\partial \mathcal{L}(\mathbf{a}^*)}{\partial \mathbf{a}^*}\right)^T$ since $J_{g_{\theta}}^T\vert_{\mathbf{a}^*}=J_{f_{\theta}}^T\vert_{\mathbf{a}^*}-I$, and it converges with linear convergence rate as long as $\lVert J_{f_{\theta}}^T\vert_{\mathbf{a}^*} \rVert< 1$. In this way, gradients for the parameters can be calculated only with the equilibrium state and equation. 

\section{Proposed IDE Method}\label{sec:method}

In this section, we first derive the equilibrium state of FSNNs under both continuous and discrete views, and demonstrate that FSNNs can be treated as solving a fixed-point equation. Then we introduce how to train the network by the proposed IDE method based on the equation and briefly discuss the biological plausibility. Finally, we incorporate the multi-layer structure into the model for more non-linearity and stronger representation ability.

\subsection{Derivation of Equilibrium States for Feedback Spiking Neural Networks}\label{equilibrium derivation}

\subsubsection{Continuous View}

We first consider a group of spiking neurons with feedback connections. Let $\mathbf{u}(t)$ and $\mathbf{s}(t)$ denote the membrane potentials and spikes of these neurons at time $t$ respectively, $\mathbf{x}(t)$ denote the inputs, $\mathbf{W}$ denote the feedback weight matrix, $\mathbf{F}$ denote the weight matrix from inputs to these neurons, and $\mathbf{b}$ denote a constant bias. Under the IF model, the dynamics of membrane potentials are expressed as:
\begin{equation}
    \frac{\mathrm{d}\mathbf{u}}{\mathrm{d}t}=\mathbf{W}\mathbf{s}(t-\Delta t_d) + \mathbf{F}\mathbf{x}(t) + \mathbf{b} - V_{th}\mathbf{s}(t),
\end{equation}
where $\Delta t_d$ is a time delay of feedback connections, and $V_{th}$ is the threshold. Note that $\mathbf{W}$ and $\mathbf{F}$ represent linear operations including both fully-connected and convolutional layers.
Define the average firing rates as $\mathbf{a}(t)=\frac{1}{t}\int_0^t \mathbf{s}(\tau)\mathrm{d}\tau$, and the average inputs as $\mathbf{\overline{x}}(t)=\frac{1}{t}\int_0^t \mathbf{x}(\tau)\mathrm{d}\tau$. Then through integration, we have:
\begin{equation}
    \mathbf{a}(t) = \frac{1}{V_{th}}\left(\frac{t-\Delta t_d}{t}\mathbf{W}\mathbf{a}(t-\Delta t_d)+\mathbf{F}\mathbf{\overline{x}}(t)+\mathbf{b}-\frac{\mathbf{u}(t)}{t}\right).
    \label{eq.con1}
\end{equation}

Eq.(\ref{eq.con1}) roughly follows a fixed-point update scheme except the existence of $\mathbf{u}(t)$. Now we dive into $\mathbf{u}(t)$. Since neurons will not spike when the accumulated $\mathbf{u}(t)$ is negative, if $v_i(t)=\left(\frac{t-\Delta t_d}{t}\mathbf{W}\mathbf{a}(t-\Delta t_d)+\mathbf{F}\mathbf{\overline{x}}(t)+\mathbf{b}\right)_i<0$, ideally neuron $i$ generates no spike and its accumulated negative term remains in $\mathbf{u}_i(t)$. So $\mathbf{u}_i(t)$ can be divided as $\mathbf{u}_i(t)=\mathbf{u}^-_i(t)+\mathbf{u}^+_i(t)$, where  
$\frac{1}{t}\mathbf{u}^-_i(t)=\min (v_i(t), 0)$ is the remaining negative term, 
and $\mathbf{u}^+_i(t)$ is the positive one typically bounded in the range between $0$ and $V_{th}$. There could be random error in $\mathbf{u}^+_i(t)$ in the context of random arrival of spikes rather than the average condition (e.g. the average is 0, but a large positive input followed by a negative one will generate an unexpected spike). Despite this, we can still assume $\mathbf{u}^+_i(t)$ to be bounded by a constant when average inputs converge. By this decomposition, we have the equation with the element-wise ReLU function ($\text{ReLU}(x)=
\max (x, 0)
$) and bounded $\mathbf{u}^+(t)$:
\begin{equation}
    \mathbf{a}(t) = \text{ReLU}\left(\frac{1}{V_{th}}\left(\frac{t-\Delta t_d}{t}\mathbf{W}\mathbf{a}(t-\Delta t_d)+\mathbf{F}\mathbf{\overline{x}}(t)+\mathbf{b}\right)\right)-\frac{1}{V_{th}}\frac{\mathbf{u}^+(t)}{t}.
    \label{eq.con2}
\end{equation}

With Eq.(\ref{eq.con2}), we can derive that the average firing rate will gradually converge to an equilibrium state.

\newtheorem{thm}{\bf Theorem}
\begin{thm}\label{thm1}
If the average inputs converge to an equilibrium point $\mathbf{\overline{x}}(t)\rightarrow \mathbf{x^*}$, and there exists constant $c$ and $\gamma<1$ such that $\vert\mathbf{u}^+_i(t)\vert\leq c,\forall i,t$ and $\lVert \mathbf{W} \rVert_2 \leq \gamma V_{th}$, then the average firing rates of FSNN with continuous IF model in Eq.(\ref{eq.con2}) will converge to an equilibrium point $\mathbf{a}(t)\rightarrow \mathbf{a^*}$, which satisfies the fixed-point equation $\mathbf{a^*} = \text{ReLU}\left(\frac{1}{V_{th}}\left(\mathbf{W}\mathbf{a^*}+\mathbf{F}\mathbf{x^*}+\mathbf{b}\right)\right)$.
\end{thm}

The proof can be found in Appendix C. Theorem~\ref{thm1} rigorously shows the equilibrium state under the IF model, and we can view the forward computation of FSNN as solving this fixed-point equation.

As for the LIF model, we can similarly define the weighted average firing rate $\mathbf{\hat{a}}(t)=\frac{\int_0^t \kappa(t-\tau)\mathbf{s}(\tau)\mathrm{d}\tau}{\int_0^t \kappa(t-\tau)\mathrm{d}\tau}$ and the weighted average inputs $\mathbf{\hat{x}}(t)=\frac{\int_0^t \kappa(t-\tau)\mathbf{x}(\tau)\mathrm{d}\tau}{\int_0^t \kappa(t-\tau)\mathrm{d}\tau}$, where $\kappa(\tau)=\exp(-\frac{\tau}{\tau_m})$ is the response kernel of the LIF model. In this setting, however, there could be random errors caused by $\mathbf{u}^+(t)$ as its denominator $\int_0^t \kappa(t-\tau)\mathrm{d}\tau$ does not go to infinity. We consider it as an approximate solver for the equilibrium with random errors, as shown in Proposition~\ref{pro1}. Please refer to Appendix E for details.

\newtheorem{pro}{\bf Proposition}
\begin{pro}\label{pro1}
If the weighted average inputs converge to an equilibrium point $\mathbf{\hat{x}}(t)\rightarrow \mathbf{x^*}$, and there exists constant $c$ and $\gamma<1$ such that $\vert\mathbf{u}^+_i(t)\vert\leq c,\forall i,t$ and $\lVert \mathbf{W} \rVert_2 \leq \gamma V_{th}$, then the weighted average firing rates $\mathbf{\hat{a}}(t)$ of FSNN with continuous LIF model gradually approximate an equilibrium point $\mathbf{a^*}$ with bounded random errors, which satisfies $\mathbf{a^*} = \text{ReLU}\left(\frac{1}{V_{th}}\left(\mathbf{W}\mathbf{a^*}+\mathbf{F}\mathbf{x^*}+\mathbf{b}\right)\right)$.
\end{pro}

\subsubsection{Discrete View}

In practice, we will simulate SNNs with discretization. 
Now consider the computation in Eq.(\ref{eq.discrete}). With feedback connections, the update equation of membrane potentials under the IF model is:
\begin{equation}
    \mathbf{u}[t+1] = \mathbf{u}[t] + \mathbf{W}\mathbf{s}[t] + \mathbf{F}\mathbf{x}[t] + \mathbf{b} - V_{th}\mathbf{s}[t+1],
\end{equation}
where we treat the feedback delay in one time step for simplicity.
Define the average firing rates as $\mathbf{a}[t]=\frac{1}{t}\sum_{\tau=1}^t \mathbf{s}[\tau]$, the average inputs as $\mathbf{\overline{x}}[t]=\frac{1}{t+1}\sum_{\tau=0}^t \mathbf{x}[\tau]$, and $\mathbf{u}[0]=\mathbf{0},\mathbf{s}[0]=\mathbf{0}$. By summation, we have:
\begin{equation}
    \mathbf{a}[t+1] = \frac{1}{V_{th}}\left(\frac{t}{t+1}\mathbf{W}\mathbf{a}[t]+\mathbf{F}\mathbf{\overline{x}}[t]+\mathbf{b}-\frac{\mathbf{u}[t+1]}{t+1}\right).
    \label{eq.dis1}
\end{equation}

Different from the continuous view, $\mathbf{a}[t]$ is bounded in the range of $[0, 1]$, since there could be at most $t$ spikes during $t$ time steps. 
Therefore, $\mathbf{u}_i[t]$ will maintain both the negative terms and the exceeded positive ones. 
Similarly, 
$\mathbf{u}_i[t]=\mathbf{u}^-_i[t]+\mathbf{u}^+_i[t]$, where $\frac{1}{t}\mathbf{u}^-_i[t]=
\min \left(\max \left(v_i[t]-V_{th}, 0\right), v_i[t]\right)$ is the exceeded term, and $\mathbf{u}^+_i[t]$ is assumed to be bounded by a constant as previously indicated. Then:
\begin{equation}
    \mathbf{a}[t+1] = \sigma\left(\frac{1}{V_{th}}\left(\frac{t}{t+1}\mathbf{W}\mathbf{a}[t]+\mathbf{F}\mathbf{\overline{x}}[t]+\mathbf{b}\right)\right)-\frac{1}{V_{th}}\frac{\mathbf{u}^+[t+1]}{t+1},
    \text{where}\, \sigma(x)=\left\{\begin{aligned}
    1,\quad &x>1\\
    x,\quad &0\leq x\leq1\\
    0,\quad &x<0\\
\end{aligned}\right..
    \label{eq.dis2}
\end{equation}
With Eq.(\ref{eq.dis2}), we can derive the equilibrium state under discrete view.

\begin{thm}\label{thm2}
If the average inputs converge to an equilibrium point $\mathbf{\overline{x}}[t]\rightarrow \mathbf{x^*}$, and there exists constant $c$ and $\gamma<1$ such that $\vert\mathbf{u}^+_i[t]\vert\leq c,\forall i,t$ and $\lVert \mathbf{W} \rVert_2 \leq \gamma V_{th}$, then the average firing rates of FSNN with discrete IF model in Eq.(\ref{eq.dis2}) will converge to an equilibrium point $\mathbf{a}[t]\rightarrow \mathbf{a^*}$, which satisfies the fixed-point equation $\mathbf{a^*} = \sigma\left(\frac{1}{V_{th}}\left(\mathbf{W}\mathbf{a^*}+\mathbf{F}\mathbf{x^*}+\mathbf{b}\right)\right)$.
\end{thm}

The proof can be found in Appendix C. And for the LIF model, define the weighted average firing rate $\mathbf{\hat{a}}[t]=\frac{\sum_{\tau=1}^t \lambda^{t-\tau}\mathbf{s}[\tau]}{\sum_{\tau=1}^t \lambda^{t-\tau}}$ and the weighted average inputs $\mathbf{\hat{x}}[t]=\frac{\sum_{\tau=0}^t \lambda^{t-\tau}\mathbf{x}[\tau]}{\sum_{\tau=0}^t \lambda^{t-\tau}}$, then we can similarly consider it as an approximation solver for the equilibrium state with random errors, as shown in Proposition~\ref{pro2}. Please refer to Appendix E for details.

\begin{pro}\label{pro2}
If the weighted average inputs converge to an equilibrium point $\mathbf{\hat{x}}[t]\rightarrow \mathbf{x^*}$, and there exists constant $c$ and $\gamma<1$ such that $\vert\mathbf{u}^+_i[t]\vert\leq c,\forall i,t$ and $\lVert \mathbf{W} \rVert_2 \leq \gamma V_{th}$, then the weighted average firing rates $\mathbf{\hat{a}}[t]$ of FSNN with discrete LIF model gradually approximate an equilibrium point $\mathbf{a^*}$ with bounded random errors, which satisfies $\mathbf{a^*} = \sigma\left(\frac{1}{V_{th}}\left(\mathbf{W}\mathbf{a^*}+\mathbf{F}\mathbf{x^*}+\mathbf{b}\right)\right)$.
\end{pro}

\subsection{Training of Feedback Spiking Neural Networks}

Based on the derivation in Section~\ref{equilibrium derivation}, we can view the forward computation of FSNNs as a black-box solver for the fixed-point equilibrium equations, with some errors caused by finite time steps or the LIF model. Then we assume the (weighted) average firing rates $\mathbf{a}[T]$ after $T$ time steps approximately follow the equations. We will demonstrate how to train FSNNs and its biological plausibility.

\subsubsection{Loss and Gradient Computation}\label{sec:loss}

Suppose that we simulate the SNN by $T$ time steps. Let $\mathbf{a}[T]$ denote the final (weighted) average firing rates. We configure a readout layer after these spiking neurons, which performs as a fully-connected classification layer, with the number of outputs as class numbers. We assume that these neurons will not spike or reset, and do classification based on the accumulated membrane potential. Then the outputs are equivalent as a linear transformation on $\mathbf{a}[T]$, i.e. $\mathbf{o}=\mathbf{W}^o \mathbf{a}[T]$. The loss $L$ is defined on $\mathbf{o}$ and labels $\mathbf{y}$ by commonly used loss functions  $\mathcal{L}(\mathbf{o},\mathbf{y})$, and we leverage the cross-entropy loss.

Let $f_{\theta}$ denote the function in fixed-point equation, e.g. $f_{\theta}(\mathbf{a^*}, \mathbf{x^*})=\sigma\left(\frac{1}{V_{th}}\left(\mathbf{W}\mathbf{a^*}+\mathbf{F}\mathbf{x^*}+\mathbf{b}\right)\right)$. The gradient for parameters can be calculated based on the implicit differentiation as described in Section~\ref{sec:id} by substituting $\mathbf{a^*}$ with $\mathbf{a}[T]$. Then parameters can be optimized based on common gradient descent methods, e.g. SGD~\cite{rumelhart1986learning} and its variants. A pseudocode is presented in Appendix B.

\subsubsection{Biological Plausibility of Implicit Differentiation}\label{bio-plausibility}

While implicit differentiation may seem too abstract for information propagation compared with backpropagation, we will briefly discuss the biological possibility for this calculation 
and its connection to the Hebbian learning rule~\cite{hebb2005organization}. 
Consider the equilibrium state $\mathbf{a^*}$ following $\mathbf{a^*} = f_{\theta}(\mathbf{a^*}, \mathbf{x^*}) = \text{ReLU}\left(\frac{1}{V_{th}}\left(\mathbf{W}\mathbf{a^*}+\mathbf{F}\mathbf{x^*}+\mathbf{b}\right)\right)$. Let $\mathbf{m}=f_{\theta}'(\mathbf{a^*}, \mathbf{x^*})=H\left(\frac{1}{V_{th}}\left(\mathbf{W}\mathbf{a^*}+\mathbf{F}\mathbf{x^*}+\mathbf{b}\right)\right)=H(\mathbf{a^*}), \mathbf{M}=\text{Diag}(\mathbf{m}), \widetilde{\mathbf{W}}=\mathbf{M}\mathbf{W}$, where $H(x)=\left\{\begin{aligned}
& 1,\,x>0\\
& 0,\,x\leq0\\
\end{aligned}\right.$. 
As indicated in Section~\ref{sec:id}, the gradient can be calculated as $\nabla_{\theta} \mathcal{L} = \left(\frac{\partial \mathcal{L}}{\partial\theta}\right)^T= \left(\frac{\partial f_{\theta}(\mathbf{a^*}, \mathbf{x^*})}{\partial\theta}\right)^T \left(I-\frac{1}{V_{th}}\widetilde{\mathbf{W}}^T\right)^{-1} \left(\frac{\partial\mathcal{L}}{\partial \mathbf{a^*}}\right)^T$, and we can leverage a fixed-point update scheme to solve for $\bm{\beta}^*=\left(I-\frac{1}{V_{th}}\widetilde{\mathbf{W}}^T\right)^{-1} \left(\frac{\partial\mathcal{L}}{\partial \mathbf{a^*}}\right)^T$ by iterating $\bm{\beta}^{k+1}=\frac{1}{V_{th}}\widetilde{\mathbf{W}}^T\bm{\beta}^k+\frac{\partial\mathcal{L}}{\partial \mathbf{a^*}}, \,\bm{\beta}^k\rightarrow\bm{\beta}^*$. 
This can be viewed as computing another equilibrium for these neurons: in this stage, neurons receive the inputs $\frac{\partial\mathcal{L}}{\partial \mathbf{a^*}}$ and they use the inverse directions of connections with a mask ($\mathbf{M}$ can be viewed as a mask matrix based on the firing condition in the first stage, which may be realized by some inhibition mechanisms) to compute for the equilibrium $\bm{\beta}^*$. If $\mathbf{W}$ is symmetric and $V_{th}=1$, 
it is similar to the energy-based method equilibrium propagation~\cite{scellier2017equilibrium} to use the same weight as the forward computation for a second equilibrium computation. 
Plugging $\bm{\beta}^*$ into the gradient and calculating $\left(\frac{\partial f_{\theta}(\mathbf{a^*}, \mathbf{x^*})}{\partial\theta}\right)^T$ explicitly, we have:
$\nabla_{\mathbf{W}} \mathcal{L} = \frac{1}{V_{th}}\mathbf{M}\bm{\beta}^* \mathbf{a^*}^T, \nabla_{\mathbf{F}} \mathcal{L} = \frac{1}{V_{th}}\mathbf{M}\bm{\beta}^* \mathbf{x^*}^T.$
It is interesting to find that the change of weight from neuron $j$ to neuron $i$ is proportional to the the equilibrium state of neuron $j$ in the first stage and that of neuron $i$ in the second stage, and is related to whether neuron $i$ fires in the first stage, because 
$\nabla_{\mathbf{W}_{i,j}} \mathcal{L} = \frac{1}{V_{th}}m_i\beta^*_i a^*_j, \nabla_{\mathbf{F}_{i,j}} \mathcal{L} = \frac{1}{V_{th}}m_i\beta^*_i x^*_j$. 
It is to some extent similar to the locally updated Hebbian learning rule $\Delta w_{i,j} \propto x_i x_j$ meaning that neurons wire together if they fire together~\cite{hebb2005organization}, except that we take average firing rates and some temporal information (two stages) into account. Therefore, implicit differentiation calculation is only related to two equilibrium states by the neurons with a mask possibly realized by inhibition mechanisms, and may correspond to modified locally updated rules, which is more biologically plausible than BPTT with surrogate derivatives for SNNs. Please note that `biological plausibility' here is a brief discussion in the context of the above properties, while there may be other aspects of biological implausibility as well.

\subsection{Incorporating Multi-layer Structure into The Feedback Model}

The multi-layer structure is commonly adopted in ANNs due to its stronger non-linearity and representation ability. To enhance the non-linearity of the fixed-point equilibrium equation, we propose to incorporate multi-layer structure into FSNNs as well. We configure $N$ subgroups of neurons as different layers, where the inputs have connections to the first layer, the $(l-1)$-th layer has connections to the $l$-th layer, and the last layer has feedback connections to the first layer. Let $\mathbf{u}^l(t)$ and $\mathbf{s}^l(t)$ denote the $l$-th layer, $\mathbf{x}(t)$ denote the inputs, $\mathbf{W}^1$ denote the feedback connection from the last layer to the first layer, and $\mathbf{F}^l$ denote the weight from the $(l-1)$-th layer (or input) to the $l$-th layer. The discrete update equations of membrane potentials are expressed as:
\begin{equation}
    \left\{
    \begin{aligned}
        &\mathbf{u}^1[t + 1] = \lambda \mathbf{u}^1[t] + \mathbf{W}^1\mathbf{s}^N[t] + \mathbf{F}^1\mathbf{x}[t]+\mathbf{b}^1-V_{th}\mathbf{s}^1[t+1],\\
        &\mathbf{u}^{l+1}[t + 1] = \lambda \mathbf{u}^{l+1}[t] + \mathbf{F}^{l+1}\mathbf{s}^l[t+1]+\mathbf{b}^{l+1}-V_{th}\mathbf{s}^{l+1}[t+1],\quad l=1,2,\cdots, N-1.\\
    \end{aligned}
    \right.
    \label{eq.multilayer}
\end{equation}

An illustration figure is presented in Appendix A. 
With similar definitions of average firing rates $\mathbf{a}^l[t]$ for different layers, and $\mathbf{u}_i^l[t]={\mathbf{u}_i^l}^-[t]+{\mathbf{u}_i^l}^+[t]$, we have the equilibrium state as the following.

\begin{thm}\label{thm3}
If the average inputs converge to an equilibrium point $\mathbf{\overline{x}}[t]\rightarrow \mathbf{x^*}$, and there exists constant $c$ and $\gamma<1$ such that $\vert{\mathbf{u}_i^l}^+[t]\vert\leq c,\forall i,l,t$ and $\lVert \mathbf{W}^1\rVert_2\lVert \mathbf{F}^N\rVert_2\cdots\lVert \mathbf{F}^2\rVert_2\leq \gamma V_{th}^N$, then the average firing rates of multi-layer FSNN with discrete IF model will converge to equilibrium points $\mathbf{a}^l[t]\rightarrow {\mathbf{a}^l}^*$, which satisfy the fixed-point equations ${\mathbf{a}^1}^* = f_1\left(f_N\circ\cdots\circ f_2({\mathbf{a}^1}^*), \mathbf{x^*}\right)$ and ${\mathbf{a}^{l+1}}^*=f_{l+1}({\mathbf{a}^l}^*)$, where $f_1(\mathbf{a}, \mathbf{x})=\sigma\left(\frac{1}{V_{th}}(\mathbf{W}^1\mathbf{a}+\mathbf{F}^1\mathbf{x}+\mathbf{b}^1)\right)$ and $f_{l}(\mathbf{a}) = \sigma\left(\frac{1}{V_{th}}(\mathbf{F}^{l}\mathbf{a}+\mathbf{b}^{l})\right)$.
\end{thm}

The proof can be found in Appendix D. There is also a similar proposition for the LIF model, please refer to Appendix E for details. 
We will do classification based on the (weighted) average firing rate of the last layer and calculate the implicit differentiation for the equation on $\mathbf{a}^N[T]$. The loss, solution for implicit differentiation, and optimization methods are the same as those in Section~\ref{sec:loss}.

\section{Experiments}\label{sec:exp}

In this section, we conduct extensive experiments to demonstrate the superior performance of our proposed method. 
Please refer to Appendix F for implementation details, including restrictions on the spectral norm and batch normalization, as well as training parameters. Since few previous methods leverage feedback architectures, we compare the results of our IDE method for FSNNs with most feedforward SNNs and report network structures\footnote{The notations 
are: `64C5' means a convolution with 64 output channels and kernel size 5, `s' after `64C5' means convolution with stride 2 
while `u' after that means a transposed convolution to upscale $2\times$, `P2' means average pooling with size 2, `400' means fully-connected to 400 neurons, and `F' means feedback layers. } as well as the number of neurons and parameters during computation (calculated according to the papers or released codes) for comparison. 

\subsection{MNIST and Fashion-MNIST}

We first evaluate our method on simple static image datasets including MNIST~\cite{lecun1998gradient} and Fashion-MNIST~\cite{xiao2017fashion}, and compare the results with other directly trained SNNs~\cite{lee2016training,wu2018spatio,shrestha2018slayer,jin2018hybrid,zhang2019spike,zhang2020temporal} or similar ANNs. Inputs are the same images with binary or real values at all time steps, which can be regarded as input currents~\cite{zhang2020temporal}. We leverage single-layer FSNNs, and adopt convolutional layers for MNIST while using fully-connected layers for Fashion-MNIST following \cite{zhang2019spike}. As shown in Table~\ref{mnist}, our models achieve comparable or better results with fewer neurons and parameters in a relatively small number of time steps, compared with other direct SNN training methods on feedforward or feedback architectures. Especially, our model achieves superior results on Fashion-MNIST with the similar structure in only 5 time steps. The LIF model performs slightly better than the IF model, probably because it leverages temporal information by encoding weighted average firing rates.

\begin{table} [ht]
	\centering
	\small
	\tabcolsep=1mm
	\caption{Performance on MNIST and Fashion-MNIST. Results are based on 5 runs of experiments.}
	\begin{tabular}{ccccccc}
	    \multicolumn{7}{c}{\textbf{MNIST}}\\
		\toprule[1pt]
		Method & Network structure & Time steps & Mean$\pm$Std & Best & Neurons & Params \\
		\midrule[0.5pt]
		BP~\cite{lee2016training} & 20C5-P2-50C5-P2-200 & >200 & / & 99.31\% & 33K & 518K\\
		STBP~\cite{wu2018spatio} & 15C5-P2-40C5-P2-300 & 30 & / & 99.42\% & 26K & 607K\\
		SLAYER~\cite{shrestha2018slayer} & 12C5-P2-64C5-P2 & 300 & 99.36\%$\pm$0.05\% & 99.41\% & 28K & 51K\\
		HM2BP~\cite{jin2018hybrid} & 15C5-P2-40C5-P2-300 & 400 & 99.42\%$\pm$0.11\% & 99.49\% & 26K & 607K\\
		ST-RSBP~\cite{zhang2019spike} & 15C5-P2-40C5-P2-300 & 400 & \textbf{99.57\%$\pm$0.04\%} & \textbf{99.62\%} & 26K & 607K\\
		TSSL-BP~\cite{zhang2020temporal} & 15C5-P2-40C5-P2-300 & 5 & 99.50\%$\pm$0.02\% & 99.53\% & 26K & 607K\\
		\midrule[0.5pt]
		\textbf{IDE-IF (ours)} & 64C5s (F64C5) & 30 & 99.49\%$\pm$0.04\% & 99.55\% & 13K & 229K\\
		\textbf{IDE-LIF (ours)} & 64C5s (F64C5) & 30 & 99.53\%$\pm$0.04\% & 99.59\% & 13K & 229K\\
		\bottomrule[1pt]
	\end{tabular}
	\begin{tabular}{ccccccc}
	    \multicolumn{7}{c}{\textbf{Fashion-MNIST}}\\
		\toprule[1pt]
		Method & Network structure & Time steps & Mean$\pm$Std & Best & Neurons & Params \\
		\midrule[0.5pt]
		ANN~\cite{zhang2019spike} & 512-512 & / & / & 89.01\% & 1.8K & 670K\\
		HM2BP~\cite{zhang2019spike} & 400-400 & 400 & / & 88.99\% & 1.6K & 478K\\
		TSSL-BP~\cite{zhang2020temporal} & 400-400 & 5 & 89.75\%$\pm$0.03\% & 89.80\% & 1.6K & 478K\\
		ST-RSBP~\cite{zhang2019spike} & 400 (F400) & 400 & 90.00\%$\pm$0.14\% & 90.13\% & 1.2K & 478K\\
		\midrule[0.5pt]
		\textbf{IDE-IF (ours)} & 400 (F400) & \textbf{5} & \textbf{90.04\%$\pm$0.09\%} & \textbf{90.14\%} & 1.2K & 478K\\
		\textbf{IDE-LIF (ours)} & 400 (F400) & \textbf{5} & \textbf{90.07\%$\pm$0.10\%} & \textbf{90.25\%} & 1.2K & 478K\\
		\bottomrule[1pt]
	\end{tabular}
	\label{mnist}
\end{table}

\subsection{N-MNIST}

We also evaluate our method on the neuromorphic dataset N-MNIST~\cite{orchard2015converting}, whose inputs are spikes collected by dynamic vision sensors. We follow the same data pre-possessing as \cite{zhang2020temporal} and take 30 time steps, and we can view the (weighted) average inputs gradually converge to that at the last time step. Table~\ref{nmnist} demonstrates the comparison results of our models and other directly trained models~\cite{jin2018hybrid,shrestha2018slayer,zhang2020temporal,wu2019direct}. It shows that our method can achieve satisfactory performance on neuromorphic data as well. Especially, only 30 time steps are required by our method for satisfactory performance.

\begin{table} [ht]
	\centering
	\small
	\tabcolsep=1mm
	\caption{Performance on N-MNIST. Results are based on 5 runs of experiments.}
	\begin{threeparttable}
	\begin{tabular}{ccccccc}
		\toprule[1pt]
		Method & Network structure & Time steps & Mean$\pm$Std & Best & Neurons & Params \\
		\midrule[0.5pt]
		HM2BP~\cite{jin2018hybrid} & 400-400 & 600 & 98.88\%$\pm$0.02\% & 98.88\% & 3K & 1.1M\\
		SLAYER~\cite{shrestha2018slayer} & 500-500 & 300 & 98.89\%$\pm$0.06\% & 98.95\% & 3K & 1.4M\\
		SLAYER~\cite{shrestha2018slayer} & 12C5-P2-64C5-P2 & 300 & 99.20\%$\pm$0.02\% & 99.22\% & 40K & 61K\\
		TSSL-BP~\cite{zhang2020temporal} & 12C5-P2-64C5-P2 & 30 & 99.23\%$\pm$0.05\% & 99.28\% & 40K & 61K\\
		STBP w/o NeuNorm~\cite{wu2019direct} & CNN\tnote{1} & 60 & / & 99.44\% & 414K & 17.3M\\
		\midrule[0.5pt]
		\textbf{IDE-IF (ours)} & 64C5s (F64C5) & 30 & 99.30\%$\pm$0.04\% & 99.35\% & 21K & 291K\\
		\textbf{IDE-LIF (ours)} & 64C5s (F64C5) & \textbf{30} & \textbf{99.42\%$\pm$0.04\%} & \textbf{99.47\%} & 21K & 291K\\
		\bottomrule[1pt]
	\end{tabular}
	\begin{tablenotes}
       \scriptsize
       \item[1] 128C3-128C3-P2-128C3-256C3-P2-1024
     \end{tablenotes}
	\end{threeparttable}
	\label{nmnist}
\end{table}

\subsection{CIFAR-10 and CIFAR-100}

Then we evaluate our method on more complex CIFAR-10 and CIFAR-100 datasets~\cite{krizhevsky2009learning}. We leverage multi-layer FSNNs with structures modified from AlexNet and CIFARNet proposed in \cite{wu2019direct}, as indicated in the footnote of Table~\ref{cifar}. We compare our model with SNNs converted from ANNs~\cite{sengupta2019going,deng2021optimal,rathi2019enabling,yan2021near} and directly trained SNNs~\cite{wu2019direct,zhang2020temporal,lee2020enabling,wu2021training}. For CIFAR-100, no result of directly trained SNN is reported, so we only compare with the converted ones. Table~\ref{cifar} demonstrates the superior results of our directly trained models with fewer neurons and parameters in a small number of time steps. Especially, our model can outperform the state-of-the-art SNN performance on CIFAR-100 with only 30 time steps, and achieves 1.59\% accuracy improvement when 100 time steps are adopted. Please refer to Appendix G for more comparison results between IF and LIF models.

\begin{table} [ht]
	\centering
	\small
	\tabcolsep=1mm
	\caption{Performance on CIFAR-10 and CIFAR-100. Results are based on 5 runs of experiments.}
	\begin{center}
	\begin{threeparttable}
	\begin{tabular}{ccccccc}
	    \multicolumn{7}{c}{\textbf{CIFAR-10}}\\
		\toprule[1pt]
		Method & Network structure & Time steps & Mean$\pm$Std & Best & Neurons & Params \\
		\midrule[0.5pt]
		ANN-SNN~\cite{deng2021optimal} & CIFARNet & 400-600 & / & 90.61\% & 726K & 45M\\
		ANN-SNN~\cite{sengupta2019going} & VGG-16 & 2500 & / & 91.55\% & 311K & 15M\\
		ANN-SNN~\cite{deng2021optimal} & VGG-16 & 400-600 & / & 92.26\% & 318K & 40M\\
		Hybrid Training~\cite{rathi2019enabling} & VGG-16 & 100 & / & 91.13\% & 318K & 40M\\
		\midrule[0.5pt]
		STBP~\cite{wu2019direct} & AlexNet & 12 & / & 85.24\% & 595K & 21M\\
		TSSL-BP~\cite{zhang2020temporal} & AlexNet & 5 & 88.98\%$\pm$0.27\% & 89.22\% & 595K & 21M\\
		STBP~\cite{wu2019direct} & CIFARNet & 12 & / & 90.53\% & 726K & 45M\\
		TSSL-BP~\cite{zhang2020temporal} & CIFARNet & 5 & / & 91.41\% & 726K & 45M\\
		Surrogate gradient~\cite{lee2020enabling} & VGG-9 & 100 & / & 90.45\% & 274K & 5.9M\\
		ASF-BP~\cite{wu2021training} & VGG-7 & 400 & / & 91.35\% & >240K & >30M\\
		\midrule[0.5pt]
		\textbf{IDE-LIF (ours)} & AlexNet-F & 30 & 91.74\%$\pm$0.09\% & 91.92\% & 159K & 3.7M\\
		\textbf{IDE-LIF (ours)} & AlexNet-F & 100 & \textbf{92.03\%$\pm$0.07\%} & \textbf{92.15\%} & \textbf{159K} & \textbf{3.7M}\\
		\textbf{IDE-LIF (ours)} & CIFARNet-F & 30 & 92.08\%$\pm$0.14\% & 92.23\% & 232K & 11.8M\\
		\textbf{IDE-LIF (ours)} & CIFARNet-F & 100 & \textbf{92.52\%$\pm$0.17\%} & \textbf{92.82\%} & \textbf{232K} & \textbf{11.8M}\\
		\bottomrule[1pt]
	\end{tabular}
	\begin{tablenotes}
       \scriptsize
       \item[1] AlexNet~\cite{wu2019direct}: 96C3-256C3-P2-384C3-P2-384C3-256C3-1024-1024
       \item[2] AlexNet-F: 96C3s-256C3-384C3s-384C3-256C3 (F96C3u)
       \item[3] CIFARNet~\cite{wu2019direct}: 128C3-256C3-P2-512C3-P2-1024C3-512C3-1024-512
       \item[4] CIFARNet-F: 128C3s-256C3-512C3s-1024C3-512C3 (F128C3u)
    \end{tablenotes}
	\end{threeparttable}
    \begin{tabular}{ccccccc}
	    \multicolumn{7}{c}{\textbf{CIFAR-100}}\\
		\toprule[1pt]
		Method & Network structure & Time steps & Mean$\pm$Std & Best & Neurons & Params \\
		\midrule[0.5pt]
		ANN~\cite{sengupta2019going} & VGG-16 & / & / & 71.22\% & 311K & 15M\\
		ANN-SNN~\cite{sengupta2019going} & VGG-16 & 2500 & / & 70.77\% & 311K & 15M\\
		ANN-SNN~\cite{deng2021optimal} & VGG-16 & 400-600 & / & 70.55\% & 318K & 40M\\
		ANN-SNN~\cite{yan2021near} & VGG-* & 300 & / & 71.84\% & 540K & 9.7M\\
		\midrule[0.5pt]
		\textbf{IDE-IF (ours)} & CIFARNet-F & \textbf{30} & \textbf{71.56\%$\pm$0.31\%} & \textbf{72.10\%} & 232K & 14.8M\\
		\textbf{IDE-IF (ours)} & AlexNet-F & \textbf{100} & \textbf{72.02\%$\pm$0.16\%} & \textbf{72.23\%} & \textbf{159K} & \textbf{5.2M}\\
		\textbf{IDE-IF (ours)} & CIFARNet-F & \textbf{100} & \textbf{73.07\%$\pm$0.21\%} & \textbf{73.43\%} & 232K & 14.8M\\
		\bottomrule[1pt]
	\end{tabular}
	\end{center}
	\label{cifar}
\end{table}

\subsection{Convergence to the Equilibrium}

To verify the convergence of FSNNs to equilibrium states, we plot the difference norm on the fixed-point equation at each time step, i.e. $\lVert f_{\theta}(\mathbf{a}[t])-\mathbf{a}[t]\rVert$, where $x=f_{\theta}(x)$ is the fixed-point equation and $\mathbf{a}[t]$ is the (weighted) average firing rate at time step $t$. Figure~\ref{fig:convergence} demonstrates the convergence of different models. It is almost the same among different samples. Since the numerical precision of firing rates is only $\frac{1}{t}$, there would be a certain convergence error due to the finite time steps. And for the LIF model, there could be random errors compared with the IF model, as indicated in Section~\ref{sec:method}. The results conform to the theorems as the difference norm gradually decreases, i.e. firing rates converge to the equilibrium following the equation. It also indicates that exact precision is not necessary for satisfactory performance. 
Networks with fewer neurons converge faster, so a smaller number of time steps is needed, explaining why only 5 time steps are enough in the Fashion-MNIST experiment. 
For results on more datasets and different time steps, please refer to Appendix G.

\begin{figure}
    \centering
    \subfigure[MNIST: 64C5s (F64C5)]{
    \includegraphics[scale=0.272]{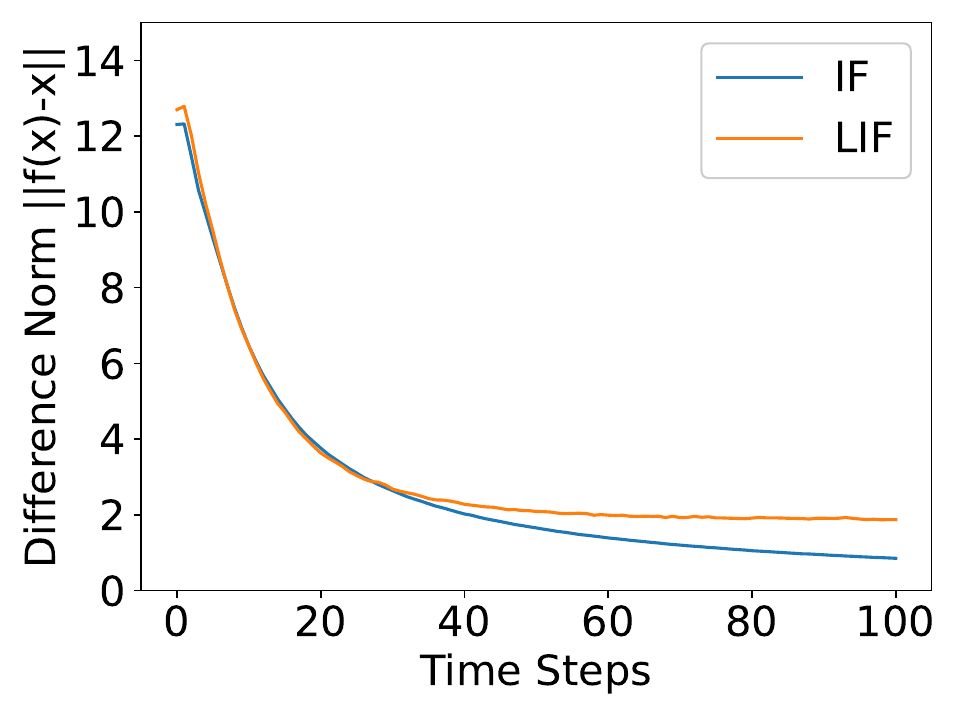}
    \label{convergence_mnist}
    }
    \subfigure[Fashion-MNIST: 400 (F400)]{
    \includegraphics[scale=0.272]{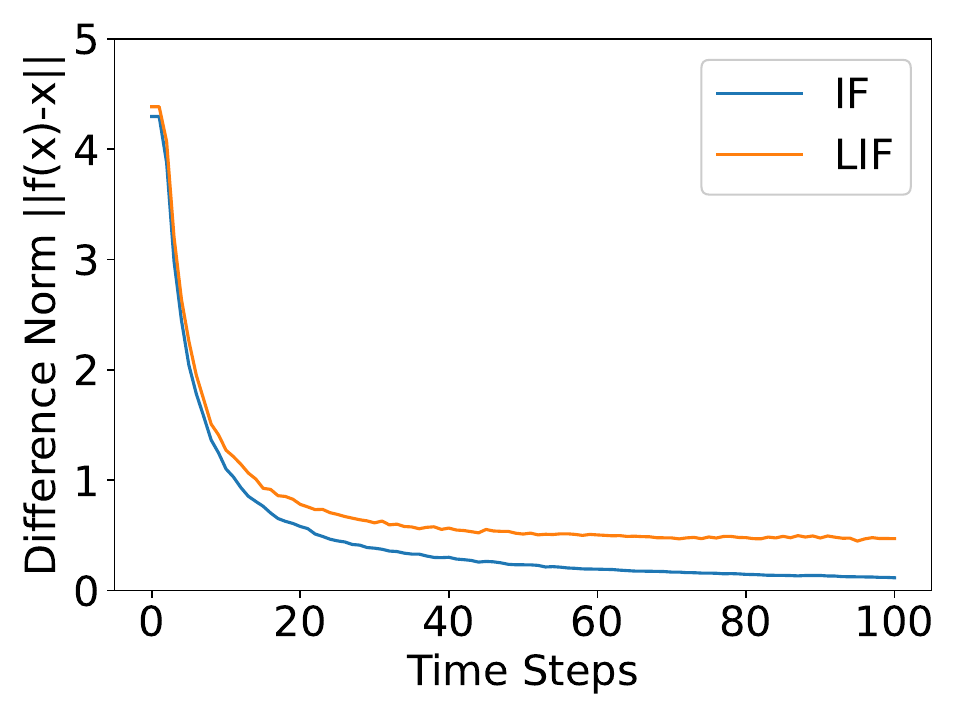}
    \label{convergence_fashionmnist}
    }
    \subfigure[CIFAR-10: CIFARNet-F]{
    \includegraphics[scale=0.272]{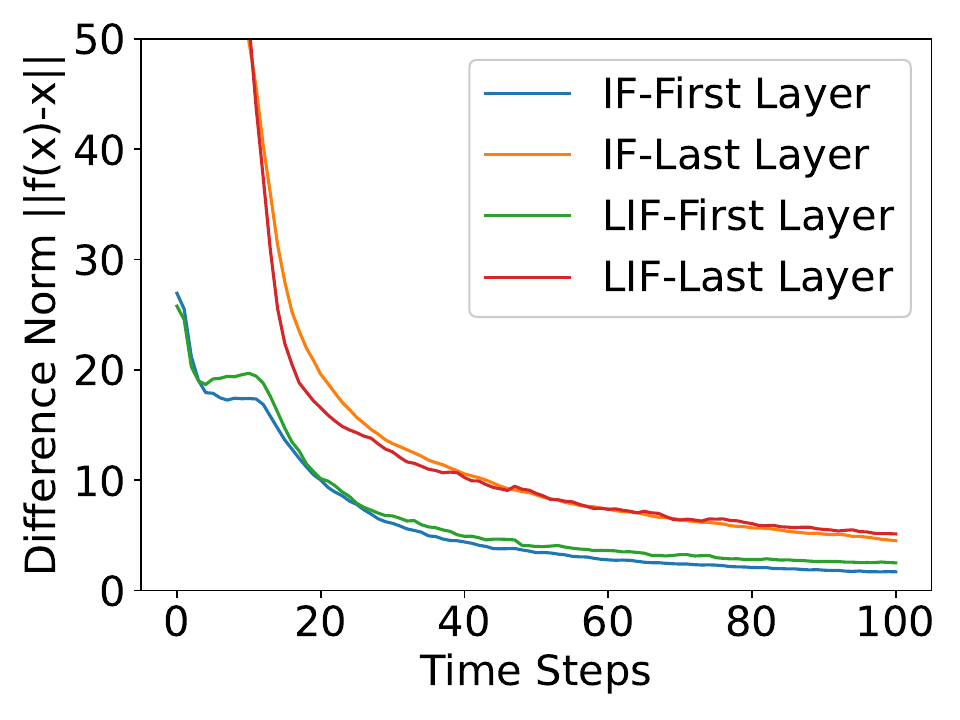}
    \label{convergence_cifar}
    }
    \caption{Convergence to the equilibrium of different models on a random sample in 100 time steps.}
    \label{fig:convergence}
\end{figure}

\subsection{Training Memory Costs}

As described in the Introduction, an important advantage of our method is that we can avoid the large memory costs, from which the methods that backpropagate along the computational graph would suffer. To quantify this ease of training, we compare the GPU memory costs of our method and the representative STBP method~\cite{wu2018spatio,wu2019direct}. The architecture and training settings are the same, and the results are shown in Table~\ref{memory costs}. It well illustrates the smaller memory costs of our method, which is also agnostic to time steps. Meanwhile, our method could achieve higher performance.

\subsection{Firing Sparsity}

As for the efficient neuromorphic computation, the firing rate is an important statistic since the energy consumption is proportional to the number of spikes. We calculate the average firing rate of trained models, and compare the IF and LIF model trained by our IDE method, as well as the LIF model trained by STBP method~\cite{wu2018spatio,wu2019direct} with the same structure. The results in Table~\ref{firing rate} demonstrate the firing sparsity of our model, as the average firing rate is only around or less than 0.7\%. And it shows that the LIF model has a slightly sparser response compared with the IF model. We note that TSSL-BP~\cite{zhang2020temporal} also reported the statistics about firing rate. According to their results, their trained model on CIFAR-10 has a roughly total 9.86\% firing rate within 5 time steps. So it is interesting to find that our models have fewer spikes than theirs, even if we have more time steps (30 vs. 5), not to mention that our models have fewer neurons. The results also show that the model trained by our method has sparser spikes compared with STBP, demonstrating the superiority of our method.

\begin{minipage}{0.47\linewidth}
\newcommand{\tabincell}[2]{\begin{tabular}{@{}#1@{}}#2\end{tabular}}
	\centering
	\small
	\tabcolsep=0.5mm
	\captionof{table}{Comparison of training memory costs and accuracy between training methods. The model is trained on CIFAR-10 with AlexNet-F structure and LIF model.}
	\begin{threeparttable}
	\begin{tabular}{cccc}
		\toprule[1pt]
		Method & Time steps & Accuracy & GPU memory\\
		\midrule[0.5pt]
		IDE (ours) & 30 & 91.74\%$\pm$0.09\% & 2.8G\\
		STBP\tnote{*} & 30 & 87.18\% & 11G\\
		IDE (ours) & 100 & 92.03\%$\pm$0.07\% & 2.8G\\
		STBP\tnote{*} & 100 & / & \tabincell{c}{out of memory\\($\approx$36G)}\\
		\bottomrule[1pt]
	\end{tabular}
	\begin{tablenotes}
       \scriptsize
       \item[*] Our implementation
     \end{tablenotes}
	\end{threeparttable}
	\label{memory costs}
\end{minipage}
\hspace{10mm}
\begin{minipage}{0.45\linewidth}
	\centering
	\small
	\tabcolsep=2mm
	\captionof{table}{The average firing rate of trained models. The model is trained on CIFAR-10 with AlexNet-F structure and 30 time steps. }
	\begin{threeparttable}
	\begin{tabular}{cccc}
		\toprule[1pt]
		Layer & IDE-IF & IDE-LIF & STBP-LIF\tnote{*}\\
		\midrule[0.5pt]
		Layer 1 & 0.0172 & 0.0166 & 0.0190\\
		Layer 2 & 0.0041 & 0.0039 & 0.0082\\
		Layer 3 & 0.0025 & 0.0024 & 0.0113\\
		Layer 4 & 0.0008 & 0.0008 & 0.0055\\
		Layer 5 & 0.0200 & 0.0177 & 0.0108\\
		\midrule[0.5pt]
		Total & 0.0070 & \textbf{0.0066} & 0.0102\\
		\bottomrule[1pt]
	\end{tabular}
	\begin{tablenotes}
       \scriptsize
       \item[*] Our implementation
     \end{tablenotes}
	\end{threeparttable}
	\label{firing rate}
\end{minipage}

\section{Conclusion}

In this work, we propose a novel training method for feedback spiking neural networks based on implicit differentiation on the equilibrium state. We first derive the equilibrium states of (weighted) average firing rates for the IF and the LIF models of FSNNs under both continuous and discrete views. Then we 
propose to optimize parameters of FSNNs only based on the implicit differentiation on the underlying fixed-point equation. This enables the backward procedure to be decoupled from the forward computational graph and therefore avoids the common training problems for SNNs, such as non-differentiability and large memory costs. 
Meanwhile, we briefly discuss the biological plausibility for the calculation of implicit differentiation, which only requires computing another equilibrium and is related to the locally updated Hebbian learning rule. 
Extensive experiments demonstrate the superior results of our method and models with fewer neurons and parameters in a small number of time steps, and the spikes are sparser in our trained models as well.

\section*{Acknowledgement} 
Z. Lin was supported by the NSF China under Grants 61625301 and 61731018, Project 2020BD006 supported by PKU-Baidu Fund, and Zhejiang Lab (grant no. 2019KB0AB02). Yisen Wang is partially supported by the National Natural Science Foundation of China under Grant 62006153, and Project 2020BD006 supported by PKU-Baidu Fund.

\small
\bibliographystyle{plain}
\bibliography{arxiv}

\normalsize
\newpage
\appendix

\section{Illustration Figure of Network Structures}

Figure~\ref{fig:illustrate} illustrates our feedback models with single-layer and multi-layer structure as indicated in Sections 4.1 and 4.3.

\begin{figure}[h]
    \centering
    \includegraphics[width=0.5\textwidth]{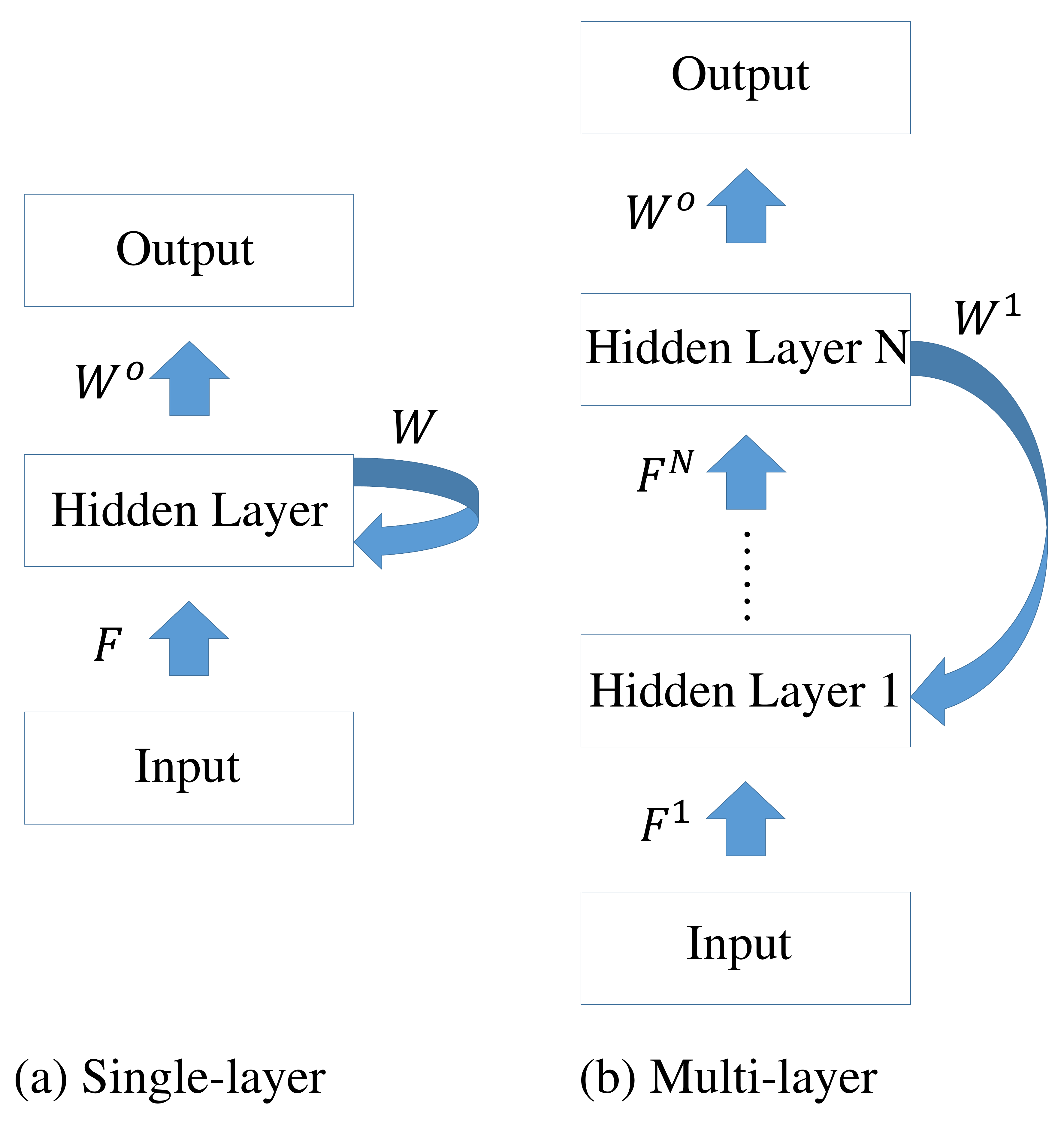}
    \caption{Illustration of feedback models with single-layer and multi-layer structure.}
    \label{fig:illustrate}
\end{figure}

\section{Pseudocode for the IDE Training Algorithm}

We present the pseudocode of one iteration of IDE training in Algorithm~\ref{algorithm: ide} to better illustrate our training method.

\begin{algorithm}[h]
    \caption{One iteration of IDE training.}
    \hspace*{0.02in} {\bf Input:}
    Network parameters $\theta$; Input data $x$; Label $y$; Time steps $T$; Other hyperparameters;\\
    \hspace*{0.02in} {\bf Output:}
    Trained network parameters $\theta$.\\
    \hspace*{0.02in} {\bf Forward:}
    \begin{algorithmic}[1]
    \State Simulate the SNN by T time steps with input $x$ based on Eq. (2) and calculate the final (weighted) average firing rate $a[T]$;
    \State Calculate the output $o$ and the loss $L$ based on $o$ and $y$.
    \end{algorithmic}
    \hspace*{0.02in} {\bf Backward:}
    \begin{algorithmic}[1]
    \State Specify the fixed-point equation $a=f_{\theta}(a)$ of the equilibrium state (define $g_{\theta}(a)=f_{\theta}(a) - a$);
    \State Calculate the gradients based on implicit differentiation:
    \State \hspace*{0.1in} (1) Solve the equation $\left(J_{g_{\theta}}^T\vert_{a[T]}\right)x^T+\left(\frac{\partial \mathcal{L}}{\partial a[T]}\right)^T=0$ by root-finding methods;
    \State \hspace*{0.1in} (2) Calculate gradients $ \frac{\partial \mathcal{L}}{\partial \theta} = -\frac{\partial \mathcal{L}}{\partial a[T]} \left(J_{g_{\theta}}^{-1}\vert_{a[T]}\right) \frac{\partial f_{\theta}(a[T])}{\partial \theta}$ based on the solution and $\frac{\partial f_{\theta}(a[T])}{\partial \theta}$;
    \State Update $\theta$ based on the gradient-based optimizer.
    \end{algorithmic}
    \label{algorithm: ide}
\end{algorithm}

\section{Proof of Theorem~\ref{thm1sup} and Theorem~\ref{thm2sup}}

We first prove Theorem~\ref{thm1sup}. Then Theorem~\ref{thm2sup} is similarly proved.

\newtheorem{thmsup}{\bf Theorem}
\begin{thmsup}\label{thm1sup}
If the average inputs converge to an equilibrium point $\mathbf{\overline{x}}(t)\rightarrow \mathbf{x^*}$, and there exists constant $c$ and $\gamma<1$ such that $\vert\mathbf{u}^+_i[t]\vert\leq c,\forall i,t$ and $\lVert \mathbf{W} \rVert_2 \leq \gamma V_{th}$, then the average firing rates of FSNN with continuous IF model will converge to an equilibrium point $\mathbf{a}(t)\rightarrow \mathbf{a^*}$, which satisfies the fixed-point equation $\mathbf{a^*} = \text{ReLU}\left(\frac{1}{V_{th}}\left(\mathbf{W}\mathbf{a^*}+\mathbf{F}\mathbf{x^*}+\mathbf{b}\right)\right)$.
\end{thmsup}

\begin{proof}

$\forall\, 0\leq\epsilon_t < \Delta t_d$, we construct the sequence $\{\mathbf{a}_{\epsilon_t}^i\}_{i=0}^{\infty}$ where $\mathbf{a}_{\epsilon_t}^i=\mathbf{a}\left(t_{\epsilon_t}^i\right), t_{\epsilon_t}^i = i\Delta t_d+\epsilon_t$. Then the equation~(\ref{eq.con2sup}) turns into the iterative equation~(\ref{eq.con3}) as the following:
\begin{equation}
    \mathbf{a}(t) = \text{ReLU}\left(\frac{1}{V_{th}}\left(\frac{t-\Delta t_d}{t}\mathbf{W}\mathbf{a}(t-\Delta t_d)+\mathbf{F}\mathbf{\overline{x}}(t)+\mathbf{b}\right)\right)-\frac{1}{V_{th}}\frac{\mathbf{u}^+(t)}{t},
    \label{eq.con2sup}
\end{equation}
\begin{equation}
    \mathbf{a}_{\epsilon_t}^{i+1} = \text{ReLU}\left(\frac{1}{V_{th}}\left(\frac{t_{\epsilon_t}^{i+1}-\Delta t_d}{t_{\epsilon_t}^i}\mathbf{W}\mathbf{a}_{\epsilon_t}^i+\mathbf{F}\mathbf{\overline{x}}(t_{\epsilon_t}^{i+1})+\mathbf{b}\right)\right)-\frac{1}{V_{th}}\frac{\mathbf{u}^+(t_{\epsilon_t}^{i+1})}{t_{\epsilon_t}^{i+1}}.
    \label{eq.con3}
\end{equation}
We prove the sequence $\{\mathbf{a}_{\epsilon_t}^i\}_{i=0}^{\infty}$ converges. Consider $\lVert \mathbf{a}_{\epsilon_t}^{i+1}-\mathbf{a}_{\epsilon_t}^i\rVert$, it satisfies:
\begin{equation}
\begin{aligned}
&\left\lVert \mathbf{a}_{\epsilon_t}^{i+1}-\mathbf{a}_{\epsilon_t}^i\right\rVert\\
=\quad&\left\lVert\, \left(\text{ReLU}\left(\frac{1}{V_{th}}\left(\frac{t_{\epsilon_t}^{i+1}-\Delta t_d}{t_{\epsilon_t}^{i+1}}\mathbf{W}\mathbf{a}_{\epsilon_t}^i+\mathbf{F}\mathbf{\overline{x}}(t_{\epsilon_t}^{i+1})+\mathbf{b}\right)\right)-\frac{1}{V_{th}}\frac{\mathbf{u}^+(t_{\epsilon_t}^{i+1})}{t_{\epsilon_t}^{i+1}}\right)\right.\\ 
&\left.- \left(\text{ReLU}\left(\frac{1}{V_{th}}\left(\frac{t_{\epsilon_t}^{i}-\Delta t_d}{t_{\epsilon_t}^{i}}\mathbf{W}\mathbf{a}_{\epsilon_t}^{i-1}+\mathbf{F}\mathbf{\overline{x}}(t_{\epsilon_t}^{i})+\mathbf{b}\right)\right)-\frac{1}{V_{th}}\frac{\mathbf{u}^+(t_{\epsilon_t}^{i})}{t_{\epsilon_t}^{i}}\right) \right\rVert\\
\leq\quad&\left\lVert\, \text{ReLU}\left(\frac{1}{V_{th}}\left(\mathbf{W}\mathbf{a}_{\epsilon_t}^i+\mathbf{F}\mathbf{x^*}+\mathbf{b}\right)\right) - \text{ReLU}\left(\frac{1}{V_{th}}\left(\mathbf{W}\mathbf{a}_{\epsilon_t}^{i-1}+\mathbf{F}\mathbf{x^*}+\mathbf{b}\right)\right) \right\rVert\\
&+ \left\lVert\, \text{ReLU}\left(\frac{1}{V_{th}}\left(\frac{t_{\epsilon_t}^{i+1}-\Delta t_d}{t_{\epsilon_t}^{i+1}}\mathbf{W}\mathbf{a}_{\epsilon_t}^i+\mathbf{F}\mathbf{\overline{x}}(t_{\epsilon_t}^{i+1})+\mathbf{b}\right)\right)-\frac{1}{V_{th}}\frac{\mathbf{u}^+(t_{\epsilon_t}^{i+1})}{t_{\epsilon_t}^{i+1}} - \text{ReLU}\left(\frac{1}{V_{th}}\left(\mathbf{W}\mathbf{a}_{\epsilon_t}^i+\mathbf{F}\mathbf{x^*}+\mathbf{b}\right)\right) \right\rVert\\ 
&+ \left\lVert\, \text{ReLU}\left(\frac{1}{V_{th}}\left(\frac{t_{\epsilon_t}^{i}-\Delta t_d}{t_{\epsilon_t}^{i}}\mathbf{W}\mathbf{a}_{\epsilon_t}^{i-1}+\mathbf{F}\mathbf{\overline{x}}(t_{\epsilon_t}^{i})+\mathbf{b}\right)\right)-\frac{1}{V_{th}}\frac{\mathbf{u}^+(t_{\epsilon_t}^{i})}{t_{\epsilon_t}^{i}} - \text{ReLU}\left(\frac{1}{V_{th}}\left(\mathbf{W}\mathbf{a}_{\epsilon_t}^{i-1}+\mathbf{F}\mathbf{x^*}+\mathbf{b}\right)\right) \right\rVert\\
\leq\quad&\left\lVert\, \text{ReLU}\left(\frac{1}{V_{th}}\left(\mathbf{W}\mathbf{a}_{\epsilon_t}^i+\mathbf{F}\mathbf{x^*}+\mathbf{b}\right)\right) - \text{ReLU}\left(\frac{1}{V_{th}}\left(\mathbf{W}\mathbf{a}_{\epsilon_t}^{i-1}+\mathbf{F}\mathbf{x^*}+\mathbf{b}\right)\right) \right\rVert\\
&+ \frac{1}{V_{th}}\left(\left\lVert \frac{\Delta t_d}{t_{\epsilon_t}^{i+1}}\mathbf{W}\mathbf{a}_{\epsilon_t}^i \right\rVert + \left\lVert \mathbf{F}\left(\mathbf{\overline{x}}(t_{\epsilon_t}^{i+1})-\mathbf{x^*}\right) \right\rVert + \left\lVert \frac{\mathbf{u}^+(t_{\epsilon_t}^{i+1})}{t_{\epsilon_t}^{i+1}} \right\rVert +  \left\lVert \frac{\Delta t_d}{t_{\epsilon_t}^i}\mathbf{W}\mathbf{a}_{\epsilon_t}^{i-1} \right\rVert + \left\lVert \mathbf{F}\left(\mathbf{\overline{x}}(t_{\epsilon_t}^{i})-\mathbf{x^*}\right) \right\rVert + \left\lVert \frac{\mathbf{u}^+(t_{\epsilon_t}^{i})}{t_{\epsilon_t}^{i}} \right\rVert \right).
\end{aligned}
\end{equation}

As $\lVert \mathbf{W} \lVert_2 \leq \gamma V_{th}, \gamma<1, \mathbf{\overline{x}}(t)\rightarrow \mathbf{x^*}$, and $\vert\mathbf{u}^+_i(t)\vert\leq c,\forall i,t$, we have $\left\lVert \mathbf{a}_{\epsilon_t}^{i+1} \right\rVert\leq \gamma \left\lVert \mathbf{a}_{\epsilon_t}^i \right\rVert+\frac{1}{V_{th}}\left(\left\lVert \mathbf{F}\mathbf{\overline{x}}(t_{\epsilon_t}^{i+1}) \right\rVert+\left\lVert \mathbf{b}\right\rVert+\left\lVert\frac{\mathbf{u}^+(t_{\epsilon_t}^{i+1})}{t_{\epsilon_t}^{i+1}}\right\rVert\right)$, and therefore $\left\lVert \mathbf{a}_{\epsilon_t}^i \right\rVert$ is bounded.

Since $t_{\epsilon_t}^{i}\rightarrow \infty, \mathbf{\overline{x}}(t)\rightarrow \mathbf{x^*}$, and $\vert\mathbf{u}^+_i(t)\vert\leq c,\forall i,t$, then $\forall \epsilon>0, \exists M$ such that when $i>M$, we have: \\
\begin{align}
    &\frac{1}{V_{th}}\left(\left\lVert \frac{\Delta t_d}{t_{\epsilon_t}^{i+1}}\mathbf{W}\mathbf{a}_{\epsilon_t}^i \right\rVert + \left\lVert \mathbf{F}\left(\mathbf{\overline{x}}(t_{\epsilon_t}^{i+1})-\mathbf{x^*}\right) \right\rVert + \left\lVert \frac{\mathbf{u}^+(t_{\epsilon_t}^{i+1})}{t_{\epsilon_t}^{i+1}} \right\rVert +  \left\lVert \frac{\Delta t_d}{t_{\epsilon_t}^i}\mathbf{W}\mathbf{a}_{\epsilon_t}^{i-1} \right\rVert + \left\lVert \mathbf{F}\left(\mathbf{\overline{x}}(t_{\epsilon_t}^{i})-\mathbf{x^*}\right) \right\rVert\right. \notag\\
    &\left. + \left\lVert \frac{\mathbf{u}^+(t_{\epsilon_t}^{i})}{t_{\epsilon_t}^{i}} \right\rVert \right) \leq \frac{\epsilon(1-\gamma)}{2}.
\end{align}

And since $\lVert \mathbf{W} \lVert_2 \leq \gamma V_{th}$, we have: \\
\begin{align}
    \left\lVert\, \text{ReLU}\left(\frac{1}{V_{th}}\left(\mathbf{W}\mathbf{a}_{\epsilon_t}^i+\mathbf{F}\mathbf{x^*}+\mathbf{b}\right)\right) - \text{ReLU}\left(\frac{1}{V_{th}}\left(\mathbf{W}\mathbf{a}_{\epsilon_t}^{i-1}+\mathbf{F}\mathbf{x^*}+\mathbf{b}\right)\right) \right\rVert \leq \gamma\left\lVert \mathbf{a}_{\epsilon_t}^i-\mathbf{a}_{\epsilon_t}^{i-1} \right\rVert.
\end{align}

Therefore, when $i>M$ it holds that:
\begin{equation}
    \left\lVert \mathbf{a}_{\epsilon_t}^{i+1}-\mathbf{a}_{\epsilon_t}^i\right\rVert \leq \gamma\left\lVert \mathbf{a}_{\epsilon_t}^i-\mathbf{a}_{\epsilon_t}^{i-1} \right\rVert + \frac{\epsilon(1-\gamma)}{2}.
\end{equation}

By iterating the above inequality, we have $ \lVert \mathbf{a}_{\epsilon_t}^{i+1}-\mathbf{a}_{\epsilon_t}^i\rVert \leq \gamma^{i-M}\lVert \mathbf{a}_{\epsilon_t}^{M+1}-\mathbf{a}_{\epsilon_t}^{M} \rVert + \frac{\epsilon(1-\gamma)}{2}\left(1+\gamma+\cdots+\gamma^{i-M-1}\right)<\gamma^{i-M}\lVert \mathbf{a}_{\epsilon_t}^{M+1}-\mathbf{a}_{\epsilon_t}^{M} \rVert + \frac{\epsilon}{2}.$ There exists $M'$ such that when $i>M+M'$, $\gamma^{i-M}\lVert \mathbf{a}_{\epsilon_t}^{M+1}-\mathbf{a}_{\epsilon_t}^{M} \rVert \leq \frac{\epsilon}{2}$, and therefore $\lVert \mathbf{a}_{\epsilon_t}^{i+1}-\mathbf{a}_{\epsilon_t}^i\rVert < \epsilon$.
According to Cauchy's convergence test, the sequence $\{\mathbf{a}_{\epsilon_t}^i\}_{i=0}^{\infty}$ converges to $\mathbf{a^*_{\epsilon_t}}$. Considering the limit, it satisfies $\mathbf{a^*_{\epsilon_t}} = \text{ReLU}\left(\frac{1}{V_{th}}\left(\mathbf{W}\mathbf{a^*_{\epsilon_t}}+\mathbf{F}\mathbf{x^*}+\mathbf{b}\right)\right)$.

The solution of $\mathbf{a}$ for the equation $\mathbf{a} = \text{ReLU}\left(\frac{1}{V_{th}}\left(\mathbf{W}\mathbf{a}+\mathbf{F}\mathbf{x^*}+\mathbf{b}\right)\right)$ is unique, since $\lVert \mathbf{W} \lVert_2 \leq \gamma V_{th}, \gamma<1$. So $\forall \epsilon_t$, the sequence $\{\mathbf{a}_{\epsilon_t}^i\}_{i=0}^{\infty}$ converges to the same point. Therefore, the average firing rates $\mathbf{a}(t)$ of IF model will converge to an equilibrium point $\mathbf{a}(t)\rightarrow \mathbf{a^*}$, which satisfies the fixed-point equation $\mathbf{a^*} = \text{ReLU}\left(\frac{1}{V_{th}}\left(\mathbf{W}\mathbf{a^*}+\mathbf{F}\mathbf{x^*}+\mathbf{b}\right)\right)$.

\end{proof}

Theorem~\ref{thm2sup} can be similarly proved as the above proof for sequence convergence, by substituting the ReLU function with $\sigma(x)=\left\{\begin{aligned}
    1,\quad &x>1\\
    x,\quad &0\leq x\leq1\\
    0,\quad &x<0\\
\end{aligned}\right.$. We omit repetitive details here.

\begin{thmsup}\label{thm2sup}
If the average inputs converge to an equilibrium point $\mathbf{\overline{x}}[t]\rightarrow \mathbf{x^*}$, and there exists constant $c$ and $\gamma<1$ such that $\vert\mathbf{u}^+_i(t)\vert\leq c,\forall i,t$ and $\lVert \mathbf{W} \rVert_2 \leq \gamma V_{th}$, then the average firing rates of FSNN with discrete IF model will converge to an equilibrium point $\mathbf{a}[t]\rightarrow \mathbf{a^*}$, which satisfies the fixed-point equation $\mathbf{a^*} = \sigma\left(\frac{1}{V_{th}}\left(\mathbf{W}\mathbf{a^*}+\mathbf{F}\mathbf{x^*}+\mathbf{b}\right)\right)$.
\end{thmsup}

\section{Proof of Theorem~\ref{thm3sup}}

\begin{thmsup}\label{thm3sup}
If the average inputs converge to an equilibrium point $\mathbf{\overline{x}}[t]\rightarrow \mathbf{x^*}$, and there exists constant $c$ and $\gamma<1$ such that $\vert{\mathbf{u}_i^l}^+[t]\vert\leq c,\forall i,l,t$ and $\lVert \mathbf{W}^1\rVert_2\lVert \mathbf{F}^N\rVert_2\cdots\lVert \mathbf{F}^2\rVert_2\leq \gamma V_{th}^N$, then the average firing rates of multi-layer FSNN with discrete IF model will converge to equilibrium points $\mathbf{a}^l[t]\rightarrow {\mathbf{a}^l}^*$, which satisfy the fixed-point equations ${\mathbf{a}^1}^* = f_1\left(f_N\circ\cdots\circ f_2({\mathbf{a}^1}^*), \mathbf{x^*}\right)$ and ${\mathbf{a}^{l+1}}^*=f_{l+1}({\mathbf{a}^l}^*)$, where $f_1(\mathbf{a}, \mathbf{x})=\sigma\left(\frac{1}{V_{th}}(\mathbf{W}^1\mathbf{a}+\mathbf{F}^1\mathbf{x}+\mathbf{b}^1)\right), f_{l+1}(\mathbf{a}) = \sigma\left(\frac{1}{V_{th}}(\mathbf{F}^{l+1}\mathbf{a}+\mathbf{b}^{l+1})\right)$.
\end{thmsup}

\begin{proof}

When the multi-layer structure is considered, with similar definitions of average firing rates $\mathbf{a}^l[t]$ for different layers and the separation  $\mathbf{u}_i^l[t]={\mathbf{u}_i^l}^-[t]+{\mathbf{u}_i^l}^+[t]$, we have the equations:
\begin{equation}
\left\{
\begin{aligned}
    &\mathbf{a}^1[t+1] = \sigma\left(\frac{1}{V_{th}}\left(\frac{t}{t+1}\mathbf{W}^1\mathbf{a}^N[t]+\mathbf{F}^1\mathbf{\overline{x}}[t]+\mathbf{b}^1\right)\right)-\frac{1}{V_{th}}\frac{{\mathbf{u}^1}^+[t+1]}{t+1},\\
    &\mathbf{a}^{l+1}[t+1] = \sigma\left(\frac{1}{V_{th}}\left(\mathbf{F}^{l+1}\mathbf{a}^{l}[t+1]+\mathbf{b}^{l+1}\right)\right)-\frac{1}{V_{th}}\frac{{\mathbf{u}^{l+1}}^+[t+1]}{t+1},\quad l=1,\cdots,N-1.\\
\end{aligned}
\right.
\end{equation}

Let $f_1^{t+1}(\mathbf{a}, \mathbf{x}, \mathbf{u}^+)=\sigma\left(\frac{1}{V_{th}}\left(\frac{t}{t+1}\mathbf{W}^1\mathbf{a}+\mathbf{F}^1\mathbf{x}+\mathbf{b}^1\right)\right)-\frac{1}{V_{th}}\frac{\mathbf{u}^+}{t+1}, \\f_{l+1}^{t}(\mathbf{a}, \mathbf{u}^+)=\sigma\left(\frac{1}{V_{th}}\left(\mathbf{F}^{l+1}\mathbf{a}+\mathbf{b}^{l+1}\right)\right)-\frac{1}{V_{th}}\frac{\mathbf{u}^+}{t}, \\f_1(\mathbf{a}, \mathbf{x})=\sigma\left(\frac{1}{V_{th}}(\mathbf{W}^1\mathbf{a}+\mathbf{F}^1\mathbf{x}+\mathbf{b}^1)\right), f_{l+1}(\mathbf{a}) = \sigma\left(\frac{1}{V_{th}}(\mathbf{F}^{l+1}\mathbf{a}+\mathbf{b}^{l+1})\right).$ 

Then $\mathbf{a}^1[t+1] = f_1^{t+1}\left(f_N^{t}\left(\cdots f_2^{t}\left(\mathbf{a}^1[t], {\mathbf{u}^2}^+[t]\right) \cdots, {\mathbf{u}^N}^+[t]\right), \mathbf{\overline{x}[t]}, {\mathbf{u}^1}^+[t+1]\right)$.

We have:
\begin{equation}
\begin{aligned}
& \left\lVert \mathbf{a}^1[t+1] - \mathbf{a}^1[t] \right\rVert\\
=\quad & \left\lVert f_1^{t+1}\left(f_N^{t}\left(\cdots f_2^{t}\left(\mathbf{a}^1[t], {\mathbf{u}^2}^+[t]\right) \cdots, {\mathbf{u}^N}^+[t]\right), \mathbf{\overline{x}[t]}, {\mathbf{u}^1}^+[t+1]\right)\right.\\
& \left.- f_1^{t}\left(f_N^{t-1}\left(\cdots f_2^{t-1}\left(\mathbf{a}^1[t-1], {\mathbf{u}^2}^+[t-1]\right) \cdots, {\mathbf{u}^N}^+[t-1]\right), \mathbf{\overline{x}[t-1]}, {\mathbf{u}^1}^+[t]\right) \right\rVert\\ 
\leq\quad & \left\lVert f_1\left(f_N\left(\cdots f_2\left(\mathbf{a}^1[t] \right) \cdots \right), \mathbf{x^*}\right) - f_1\left(f_N\left(\cdots f_2\left(\mathbf{a}^1[t-1] \right) \cdots \right), \mathbf{x^*}\right) \right\rVert\\
& + \left\lVert  f_1^{t+1}\left(f_N^{t}\left(\cdots f_2^{t}\left(\mathbf{a}^1[t], {\mathbf{u}^2}+[t]\right) \cdots, {\mathbf{u}^N}^+[t]\right), \mathbf{\overline{x}[t]}, {\mathbf{u}^1}^+[t+1]\right) - f_1\left(f_N\left(\cdots f_2\left(\mathbf{a}^1[t] \right) \cdots \right), \mathbf{x^*}\right) \right\rVert\\
& + \left\lVert f_1^{t}\left(f_N^{t-1}\left(\cdots f_2^{t-1}\left(\mathbf{a}^1[t-1], {\mathbf{u}^2}^+[t-1]\right) \cdots, {\mathbf{u}^N}^+[t-1]\right), \mathbf{\overline{x}[t-1]}, {\mathbf{u}^1}^+[t]\right)\right.\\
&\left.\quad - f_1\left(f_N\left(\cdots f_2\left(\mathbf{a}^1[t-1] \right) \cdots \right), \mathbf{x^*}\right) \right\rVert\\ 
\leq\quad & \left\lVert f_1\left(f_N\left(\cdots f_2\left(\mathbf{a}^1[t] \right) \cdots \right), \mathbf{x^*}\right) - f_1\left(f_N\left(\cdots f_2\left(\mathbf{a}^1[t-1] \right) \cdots \right), \mathbf{x^*}\right) \right\rVert\\
& + \frac{1}{V_{th}}\left(\left\lVert \frac{1}{t+1}\mathbf{W}^1f_N^{t}\left(\cdots f_2^{t}\left(\mathbf{a}^1[t], {\mathbf{u}^2}^+[t]\right) \cdots, {\mathbf{u}^N}^+[t]\right) \right\rVert\right.\\
& \left.\quad + \underbrace{\left\lVert \mathbf{W}^1\left(f_N^{t}\left(\cdots f_2^{t}\left(\mathbf{a}^1[t], {\mathbf{u}^2}^+[t]\right) \cdots, {\mathbf{u}^N}^+[t]\right) - f_N\left(\cdots f_2\left(\mathbf{a}^1[t] \right) \cdots \right) \right) \right\rVert}_{A} \right.\\
& \left.\quad + \left\lVert \mathbf{F}\left(\mathbf{\overline{x}}[t]-\mathbf{x^*}\right) \right\rVert + \left\lVert \frac{{\mathbf{u}^1}^+[t+1]}{t+1} \right\rVert \right.\\
& \left.\quad +  \left\lVert \frac{1}{t}\mathbf{W}^1f_N^{t-1}\left(\cdots f_2^{t-1}\left(\mathbf{a}^1[t-1], {\mathbf{u}^2}^+[t-1]\right) \cdots, {\mathbf{u}^N}^+[t-1]\right) \right\rVert\right.\\
& \left.\quad + \underbrace{\left\lVert \mathbf{W}^1\left(f_N^{t-1}\left(\cdots f_2^{t-1}\left(\mathbf{a}^1[t-1], {\mathbf{u}^2}^+[t-1]\right) \cdots, {\mathbf{u}^N}^+[t-1]\right) - f_N\left(\cdots f_2\left(\mathbf{a}^1[t-1] \right) \cdots \right) \right) \right\rVert}_{B} \right.\\
& \left.\quad + \left\lVert \mathbf{F}\left(\mathbf{\overline{x}}[t-1]-\mathbf{x^*}\right) \right\rVert + \left\lVert \frac{{\mathbf{u}^1}^+[t]}{t} \right\rVert \right).\\ 
\end{aligned}
\end{equation}

For the term $A$ and $B$, they are bounded by:
\begin{equation}
\begin{aligned}
A\leq\quad & \frac{1}{V_{th}}\left(\left\lVert \mathbf{W}^1\mathbf{F}^N \left(f_{N-1}^t\left(\cdots f_2^t\left(\mathbf{a}^1[t], {\mathbf{u}^2}^+[t]  \right)\cdots, {\mathbf{u}^{N-1}}^+ \right) - f_{N-1}\left(\cdots f_2\left(\mathbf{a}^1[t]\right)  \cdots\right) \right) \right\rVert + \left\lVert \mathbf{W}^1 \frac{{\mathbf{u}^N}^+[t]}{t} \right\rVert \right)\\
\leq\quad & \cdots\cdots\\
\leq\quad & \frac{1}{V_{th}}\left\lVert \mathbf{W}^1 \frac{{\mathbf{u}^N}^+[t]}{t} \right\rVert + \cdots + \frac{1}{V_{th}^{N-1}} \left\lVert \mathbf{W}^1\mathbf{F}^N\cdots \mathbf{F}^3\frac{{\mathbf{u}^2}^+[t]}{t} \right\rVert,
\end{aligned}
\end{equation}
and $B$ has the same form as $A$ by substituting $t$ with $t-1$.

Since $\lVert \mathbf{W}^1\rVert_2\lVert \mathbf{F}^N\rVert_2\cdots\lVert \mathbf{F}^2\rVert_2\leq \gamma V_{th}^N$, we have:
\begin{equation}
\begin{aligned}
    & \left\lVert f_1\left(f_N\left(\cdots f_2\left(\mathbf{a}^1[t] \right) \cdots \right), \mathbf{x^*}\right) - f_1\left(f_N\left(\cdots f_2\left(\mathbf{a}^1[t-1] \right) \cdots \right), \mathbf{x^*}\right) \right\rVert\\
    \leq\quad & \left\lVert \frac{1}{V_{th}}\mathbf{W}^1\left(f_N\left(\cdots f_2\left(\mathbf{a}^1[t] \right) \cdots \right) - f_N\left(\cdots f_2\left(\mathbf{a}^1[t-1] \right) \cdots \right)\right) \right\rVert\\
    \leq\quad & \cdots\cdots\\
    \leq\quad & \left\lVert \frac{1}{V_{th}^N} \mathbf{W}^1\mathbf{F}^N\cdots \mathbf{F}^2 \left(\mathbf{a}^1[t] - \mathbf{a}^1[t-1] \right) \right\rVert\\
    \leq\quad & \gamma \left\lVert \mathbf{a}^1[t] - \mathbf{a}^1[t-1] \right\rVert.
\end{aligned}
\end{equation}

And since $t\rightarrow\infty, \mathbf{\overline{x}}[t]\rightarrow \mathbf{x^*}, \vert{\mathbf{u}_i^l}^+[t]\vert\leq c,\forall i,l,t$, then $\forall\epsilon>0, \exists M$ such that when $t>M$, we have:
\begin{equation}
    \left\lVert \mathbf{a}^1[t+1] - \mathbf{a}^1[t] \right\rVert \leq \gamma \left\lVert \mathbf{a}^1[t] - \mathbf{a}^1[t-1] \right\rVert + \frac{\epsilon(1-\gamma)}{2}.
\end{equation}

Then $\lVert \mathbf{a}^1[t+1] - \mathbf{a}^1[t] \rVert < \gamma^{t-M}\lVert \mathbf{a}^1[M+1] - \mathbf{a}^1[M] \rVert+\frac{\epsilon}{2}$, and there exists $M'$ such that when $t > M + M'$, $\lVert \mathbf{a}^1[t+1] - \mathbf{a}^1[t] \rVert < \epsilon$. According to Cauchy's convergence test, $\mathbf{a}^1[t]$ converges to $\mathbf{a}^{1^*}$, which satisfies ${\mathbf{a}^1}^* = f_1\left(f_N\circ\cdots\circ f_2({\mathbf{a}^1}^*), \mathbf{x^*}\right)$. Considering the limit, $\mathbf{a}^{l+1}[t]$ converges to $\mathbf{a}^{{l+1}^*}$, which satisfies ${\mathbf{a}^{l+1}}^*=f_{l+1}({\mathbf{a}^l}^*)$.

\end{proof}

\section{Derivation for the LIF Model}

\subsection{Continuous View}

We follow the same notations as Section 4.1.1 and redefine $\mathbf{W}, \mathbf{F}, \mathbf{b}$ by absorbing $\tau_m$ into them. The dynamics of membrane potentials are expressed as:
\begin{equation}
    \frac{\mathrm{d}\mathbf{u}}{\mathrm{d}t}=-\frac{1}{\tau_m}\mathbf{u}+\mathbf{W}\mathbf{s}(t-\Delta t_d) + \mathbf{F}\mathbf{x}(t) + \mathbf{b} - V_{th}\mathbf{s}(t).
\end{equation}

Through integration, we have:
\begin{equation}
    \mathbf{u}(t) = \mathbf{W}\int_0^{t-\Delta t_d}\kappa(t-\Delta t_d - \tau)\mathbf{s}(\tau)\mathrm{d}\tau + \mathbf{F}\int_0^{t}\kappa(t-\tau)\mathbf{x}(\tau)\mathrm{d}\tau + t\mathbf{b} - V_{th} \int_0^{t}\kappa(t-\tau)\mathbf{s}(\tau)\mathrm{d}\tau,
\end{equation}
where $\kappa(\tau)=\exp(-\frac{\tau}{\tau_m})$ is the response kernel of the LIF model. Define the weighted average firing rate as $\mathbf{\hat{a}}(t)=\frac{\int_0^t \kappa(t-\tau)\mathbf{s}(\tau)\mathrm{d}\tau}{\int_0^t \kappa(t-\tau)\mathrm{d}\tau}$, and the weighted average inputs as $\mathbf{\hat{x}}(t)=\frac{\int_0^t \kappa(t-\tau)\mathbf{x}(\tau)\mathrm{d}\tau}{\int_0^t \kappa(t-\tau)\mathrm{d}\tau}$. Then we have the equation:
\begin{equation}
    \mathbf{\hat{a}}(t) = \frac{1}{V_{th}}\left(\frac{\int_0^{t-\Delta t_d} \kappa(\tau)\mathrm{d}\tau}{\int_0^{t} \kappa(\tau)\mathrm{d}\tau}\mathbf{W}\mathbf{\hat{a}}(t-\Delta t_d)+\mathbf{F}\mathbf{\hat{x}}(t)+\mathbf{b}-\frac{\mathbf{u}(t)}{\int_0^{t} \kappa(\tau)\mathrm{d}\tau}\right).
\end{equation}

Similarly, we can divide $\mathbf{u}(t)$ into two parts $\mathbf{u}_i(t)={\mathbf{u}}^-_i(t)+{\mathbf{u}}^+_i(t)$, and we have the equation with the element-wise ReLU function and a bounded $\mathbf{u}^+(t)$:
\begin{equation}
    \mathbf{\hat{a}}(t) = \text{ReLU}\left(\frac{1}{V_{th}}\left(\frac{\int_0^{t-\Delta t_d} \kappa(\tau)\mathrm{d}\tau}{\int_0^{t} \kappa(\tau)\mathrm{d}\tau}\mathbf{W}\mathbf{\hat{a}}(t-\Delta t_d)+\mathbf{F}\mathbf{\hat{x}}(t)+\mathbf{b}\right)\right)-\frac{1}{V_{th}}\frac{\mathbf{u}^+(t)}{\int_0^{t} \kappa(\tau)\mathrm{d}\tau}.
\end{equation}

As $\int_0^{t} \kappa(\tau)\mathrm{d}\tau=\tau_m\left(1-\exp(-\frac{t}{\tau_m})\right)\rightarrow \tau_m$, compared with the IF model, there could be random error caused by $\frac{\mathbf{u}^+(t)}{\int_0^{t} \kappa(\tau)\mathrm{d}\tau}$ which are not eliminated with time $t\rightarrow \infty$. Therefore when the weighted average inputs converge to an equilibrium point $\mathbf{\hat{x}}(t)\rightarrow \mathbf{x^*}$, the LIF model only gradually approximates an equilibrium with some random error, and the equilibrium state $\mathbf{a^*}$ still follows the equation $\mathbf{a^*} = \text{ReLU}\left(\frac{1}{V_{th}}\left(\mathbf{W}\mathbf{a^*}+\mathbf{F}\mathbf{x^*}+\mathbf{b}\right)\right)$. The error is:

\begin{equation}
\begin{aligned}
    e(t) =\quad &\left\lVert\text{ReLU}\left(\frac{1}{V_{th}}\left(\mathbf{W}\mathbf{\hat{a}}(t)+\mathbf{F}\mathbf{x^*}+\mathbf{b}\right)\right) - \mathbf{\hat{a}}(t) \right\rVert\\
    =\quad &\left\lVert\text{ReLU}\left(\frac{1}{V_{th}}\left(\mathbf{W}\mathbf{\hat{a}}(t)+\mathbf{F}\mathbf{x^*}+\mathbf{b}\right)\right) - \text{ReLU}\left(\frac{1}{V_{th}}\left(\frac{\int_0^{t-\Delta t_d} \kappa(\tau)\mathrm{d}\tau}{\int_0^{t} \kappa(\tau)\mathrm{d}\tau}\mathbf{W}\mathbf{\hat{a}}(t-\Delta t_d)+\mathbf{F}\mathbf{\hat{x}}(t)+\mathbf{b}\right)\right)\right.\\
    &\left.+\frac{1}{V_{th}}\frac{\mathbf{u}^+(t)}{\int_0^{t} \kappa(\tau)\mathrm{d}\tau} \right\rVert\\
    \leq\quad &\frac{1}{V_{th}}\left(\left\lVert \mathbf{W}\left(\mathbf{\hat{a}}(t)-\mathbf{\hat{a}}(t-\Delta t_d)\right) \right\rVert + \left\lVert \frac{\int_{t-\Delta t_d}^t \kappa(\tau)\mathrm{d}\tau}{\int_0^{t} \kappa(\tau)\mathrm{d}\tau} \mathbf{W}\mathbf{\hat{a}}(t-\Delta t_d) \right\rVert + \left\lVert \mathbf{F}\left(\mathbf{\hat{x}}(t)-\mathbf{x^*}\right) \right\rVert + \left\lVert \frac{\mathbf{u}^+(t)}{\int_0^{t} \kappa(\tau)\mathrm{d}\tau} \right\rVert \right).\\
\end{aligned}
\label{eq.error}
\end{equation}

When there exists a constant $\gamma<1$ such that $\lVert \mathbf{W} \rVert_2 \leq \gamma V_{th}$, we have:
\begin{equation}
\begin{aligned}
    &\left\lVert \mathbf{\hat{a}}(t)-\mathbf{\hat{a}}(t-\Delta t_d) \right\rVert \\
    \leq\quad & \left\lVert \text{ReLU}\left(\frac{1}{V_{th}}\left(\mathbf{W}\mathbf{\hat{a}}(t-\Delta t_d)+\mathbf{F}\mathbf{x^*}+\mathbf{b}\right)\right) - \text{ReLU}\left(\frac{1}{V_{th}}\left(\mathbf{W}\mathbf{\hat{a}}(t-2\Delta t_d)+\mathbf{F}\mathbf{x^*}+\mathbf{b}\right)\right) \right\rVert\\
    & + \frac{1}{V_{th}}\left(\left\lVert \frac{\int_{t-\Delta t_d}^t \kappa(\tau)\mathrm{d}\tau}{\int_{0}^t \kappa(\tau)\mathrm{d}\tau}\mathbf{W}\mathbf{\hat{a}}(t-\Delta t_d) \right\rVert + \left\lVert \mathbf{F}\left(\mathbf{\hat{x}}(t)-\mathbf{x^*}\right) \right\rVert + \left\lVert \frac{\mathbf{u}^+(t)}{\int_{0}^t \kappa(\tau)\mathrm{d}\tau} \right\rVert\right.\\
    &\left.\quad +  \left\lVert \frac{\int_{t-2\Delta t_d}^{t-\Delta t_d} \kappa(\tau)\mathrm{d}\tau}{\int_{0}^{t-\Delta t_d} \kappa(\tau)\mathrm{d}\tau}\mathbf{W}\mathbf{\hat{a}}(t-2\Delta t_d) \right\rVert + \left\lVert \mathbf{F}\left(\mathbf{\hat{x}}(t-\Delta t_d)-\mathbf{x^*}\right) \right\rVert + \left\lVert \frac{\mathbf{u}^+(t-\Delta t_d)}{\int_{0}^{t-\Delta t_d} \kappa(\tau)\mathrm{d}\tau} \right\rVert \right)\\
    \leq\quad & \gamma\left\lVert \mathbf{\hat{a}}(t-\Delta t_d)-\mathbf{\hat{a}}(t-2\Delta t_d) \right\rVert\\
    & + \frac{1}{V_{th}}\left(\left\lVert \frac{\int_{t-\Delta t_d}^t \kappa(\tau)\mathrm{d}\tau}{\int_{0}^t \kappa(\tau)\mathrm{d}\tau}\mathbf{W}\mathbf{\hat{a}}(t-\Delta t_d) \right\rVert + \left\lVert \mathbf{F}\left(\mathbf{\hat{x}}(t)-\mathbf{x^*}\right) \right\rVert + \left\lVert \frac{\mathbf{u}^+(t)}{\int_{0}^t \kappa(\tau)\mathrm{d}\tau} \right\rVert\right.\\
    &\left.\quad +  \left\lVert \frac{\int_{t-2\Delta t_d}^{t-\Delta t_d} \kappa(\tau)\mathrm{d}\tau}{\int_{0}^{t-\Delta t_d} \kappa(\tau)\mathrm{d}\tau}\mathbf{W}\mathbf{\hat{a}}(t-2\Delta t_d) \right\rVert + \left\lVert \mathbf{F}\left(\mathbf{\hat{x}}(t-\Delta t_d)-\mathbf{x^*}\right) \right\rVert + \left\lVert \frac{\mathbf{u}^+(t-\Delta t_d)}{\int_{0}^{t-\Delta t_d} \kappa(\tau)\mathrm{d}\tau} \right\rVert \right),\\
\end{aligned}
\end{equation}
and
\begin{equation}
    \frac{1}{V_{th}}\lVert \mathbf{W}\left(\mathbf{\hat{a}}(t)-\mathbf{\hat{a}}(t-\Delta t_d)\right) \rVert \leq \gamma \lVert \mathbf{\hat{a}}(t)-\mathbf{\hat{a}}(t-\Delta t_d) \rVert.
    \label{eq.term1}
\end{equation}

Since $\mathbf{u}^+(t)$ is bounded by a constant $c$, $\frac{\int_{t-\Delta t_d}^t \kappa(\tau)\mathrm{d}\tau}{\int_{0}^t \kappa(\tau)\mathrm{d}\tau}\rightarrow 0$ and $\mathbf{\hat{x}}(t)\rightarrow \mathbf{x^*}$, there exists a constant $c'$ and $M$ such that when $t>M$, the following holds:
\begin{equation}
    \lVert \mathbf{\hat{a}}(t)-\mathbf{\hat{a}}(t-\Delta t_d) \rVert \leq \gamma\lVert \mathbf{\hat{a}}(t-\Delta t_d)-\mathbf{\hat{a}}(t-2\Delta t_d) \rVert + c'.
\end{equation}

Thus $\lVert \mathbf{\hat{a}}(t)-\mathbf{\hat{a}}(t-\Delta t_d) \rVert$ is bounded by $\frac{c'}{1-\gamma}$ when $t$ is large enough. Plugging this and Eq.~(\ref{eq.term1}) into Eq.~(\ref{eq.error}), we get that the random error is bounded by a constant related with $c, V_{th}, \tau_m, \gamma$. This leads to Proposition~\ref{pro1sup}.

\newtheorem{prosup}{\bf Proposition}
\begin{prosup}\label{pro1sup}
If the weighted average inputs converge to an equilibrium point $\mathbf{\hat{x}}(t)\rightarrow \mathbf{x^*}$, and there exists constant $c$ and $\gamma<1$ such that $\vert\mathbf{u}^+_i(t)\vert\leq c,\forall i,t$ and $\lVert \mathbf{W} \rVert_2 \leq \gamma V_{th}$, then the weighted average firing rates $\mathbf{\hat{a}}(t)$ of FSNN with continuous LIF model gradually approximate an equilibrium point $mathbf{a^*}$ with bounded random errors, which satisfies $\mathbf{a^*} = \text{ReLU}\left(\frac{1}{V_{th}}\left(\mathbf{W}\mathbf{a^*}+\mathbf{F}\mathbf{x^*}+\mathbf{b}\right)\right)$.
\end{prosup}

Although there would be random error, the remaining membrane potential of the LIF model will gradually decrease if there is no positive input, which means the random error tend to be eliminated. We can still view the FSNN with LIF model as approximately solving the fixed-point equilibrium equation.

\subsection{Discrete Perspective}

The discrete update equation of membrane potentials under LIF model is:
\begin{equation}
    \mathbf{u}[t+1] = \lambda\mathbf{u}[t] + \mathbf{W}\mathbf{s}[t] + \mathbf{F}\mathbf{x}[t] + b - V_{th}\mathbf{s}[t+1].
\end{equation}

Define the weighted average firing rate as $\mathbf{\hat{a}}[t]=\frac{\sum_{\tau=1}^t \lambda^{t-\tau}\mathbf{s}[\tau]}{\sum_{\tau=1}^t \lambda^{t-\tau}}$ and the weighted average inputs as $\mathbf{\hat{x}}[t]=\frac{\sum_{\tau=0}^t \lambda^{t-\tau}\mathbf{x}[\tau]}{\sum_{\tau=0}^t \lambda^{t-\tau}}$, then through summation and consideration of the division of $\mathbf{u}[t+1]$, we have:
\begin{equation}
    \mathbf{\hat{a}}[t+1] = \frac{1}{V_{th}}\left(\frac{\sum_{i=0}^{t-1} \lambda^i}{\sum_{i=0}^{t} \lambda^i}\mathbf{W}\mathbf{\hat{a}}[t]+\mathbf{F}\mathbf{\hat{x}}[t]+\mathbf{b}-\frac{\mathbf{u}[t+1]}{\sum_{i=0}^{t} \lambda^i}\right),
\end{equation}
\begin{equation}
    \mathbf{\hat{a}}[t+1] = \sigma\left(\frac{1}{V_{th}}\left(\frac{\sum_{i=0}^{t-1} \lambda^i}{\sum_{i=0}^{t} \lambda^i}\mathbf{W}\mathbf{\hat{a}}[t]+\mathbf{F}\mathbf{\hat{x}}[t]+\mathbf{b}\right)\right)-\frac{1}{V_{th}}\frac{\mathbf{u}^+[t+1]}{\sum_{i=0}^{t} \lambda^i},
\end{equation}
where $\sigma(x)=\left\{\begin{aligned}
    1,\quad &x>1\\
    x,\quad &0\leq x\leq1\\
    0,\quad &x<0\\
\end{aligned}\right.$.

Proposition~\ref{pro2sup} is similarly derived as Proposition~\ref{pro1sup} by substituting the ReLU function with $\sigma$. We omit repetitive details here.

\begin{prosup}\label{pro2sup}
If the weighted average inputs converge to an equilibrium point $\mathbf{\hat{x}}[t]\rightarrow \mathbf{x^*}$, and there exists constant $c$ and $\gamma<1$ such that $\vert\mathbf{u}^+_i[t]\vert\leq c,\forall i,t$ and $\lVert \mathbf{W} \rVert_2 \leq \gamma V_{th}$, then the weighted average firing rates $\mathbf{\hat{a}}[t]$ of FSNN with discrete LIF model gradually approximate an equilibrium point $\mathbf{a^*}$ with bounded random errors, which satisfies $\mathbf{a^*} = \sigma\left(\frac{1}{V_{th}}\left(\mathbf{W}\mathbf{a^*}+\mathbf{F}\mathbf{x^*}+\mathbf{b}\right)\right)$.
\end{prosup}

\subsection{Multi-layer Structure}

When the multi-layer structure is considered, the discrete update equation of membrane potentials are expressed as:
\begin{equation}
    \left\{
    \begin{aligned}
        &\mathbf{u}^1[t + 1] = \lambda \mathbf{u}^1[t] + \mathbf{W}^1\mathbf{s}^N[t] + \mathbf{F}^1\mathbf{x}[t]+\mathbf{b}^1-V_{th}\mathbf{s}^1[t+1],\\
        &\mathbf{u}^{l+1}[t + 1] = \lambda \mathbf{u}^{l+1}[t] + \mathbf{F}^{l+1}\mathbf{s}^l[t+1]+\mathbf{b}^{l+1}-V_{th}\mathbf{s}^{l+1}[t+1],\quad l=1,2,\cdots, N-1\\
    \end{aligned}
    \right.
    \label{eq.multilayersup}
\end{equation}

Define the weighted average firing rates of different layers as $\mathbf{\hat{a}}^l[t]=\frac{\sum_{\tau=1}^t \lambda^{t-\tau}\mathbf{s}^l[\tau]}{\sum_{\tau=1}^t \lambda^{t-\tau}}$ and the weighted average inputs as $\mathbf{\hat{x}}[t]=\frac{\sum_{\tau=0}^t \lambda^{t-\tau}\mathbf{x}[\tau]}{\sum_{\tau=0}^t \lambda^{t-\tau}}$, then through summation and consideration of the division of $\mathbf{u}^l[t+1]$, we have:
\begin{equation}
\left\{
\begin{aligned}
    &\mathbf{\hat{a}}^1[t+1] = \sigma\left(\frac{1}{V_{th}}\left(\frac{\sum_{i=0}^{t-1} \lambda^i}{\sum_{i=0}^{t} \lambda^i}\mathbf{W}^1\mathbf{\hat{a}}^N[t]+\mathbf{F}^1\mathbf{\hat{x}}[t]+\mathbf{b}^1-\frac{{\mathbf{u}^1}^+[t+1]}{\sum_{i=0}^{t} \lambda^i}\right)\right),\\
    &\mathbf{\hat{a}}^{l+1}[t+1] = \sigma\left(\frac{1}{V_{th}}\left(\mathbf{F}^{l+1}\mathbf{\hat{a}}^{l}[t+1]+\mathbf{b}^{l+1}-\frac{{\mathbf{u}^{l+1}}^+[t+1]}{\sum_{i=0}^{t} \lambda^i}\right)\right),\quad l=1,\cdots,N-1.\\
\end{aligned}
\right.
\end{equation}

With similar techniques in the proof of Theorem~\ref{thm3sup} and Proposition~\ref{pro1sup}, we can derive the Proposition~\ref{pro3} as the following. We omit repetitive details here.

\begin{prosup}\label{pro3}
If the weighted average inputs converge to an equilibrium point $\mathbf{\hat{x}}[t]\rightarrow \mathbf{x^*}$, and there exists constant $c$ and $\gamma<1$ such that $\vert{\mathbf{u}_i^l}^+[t]\vert\leq c,\forall i,l,t$ and $\lVert \mathbf{W}^1\rVert_2\lVert \mathbf{F}^N\rVert_2\cdots\lVert \mathbf{F}^2\rVert_2\leq \gamma V_{th}^N$, then the weighted average firing rates $\mathbf{\hat{a}}^l[t]$ of multi-layer FSNN with discrete LIF model gradually approximate equilibrium points ${\mathbf{a}^l}^*$ with bounded random errors, which satisfy  ${\mathbf{a}^1}^* = f_1\left(f_N\circ\cdots\circ f_2({\mathbf{a}^1}^*), \mathbf{x^*}\right)$ and ${\mathbf{a}^{l+1}}^*=f_{l+1}({\mathbf{a}^l}^*)$, where $f_1(\mathbf{a}, \mathbf{x})=\sigma\left(\frac{1}{V_{th}}(\mathbf{W}^1\mathbf{a}+\mathbf{F}^1\mathbf{x}+\mathbf{b}^1)\right), f_{l+1}(\mathbf{a}) = \sigma\left(\frac{1}{V_{th}}(\mathbf{F}^{l+1}\mathbf{a}+\mathbf{b}^{l+1})\right)$.
\end{prosup}

\section{Implementation Details}

In this section, we describe the details for training our model. We will first introduce some operations in our model including restriction on the spectral norm and batch normalization, and then elaborate the training settings for the experiments.

\subsection{Restriction on Spectral Norm}

As indicated in the theorems and propositions, a sufficient condition for the convergence of FSNN is $\lVert \mathbf{W}\rVert_2\leq\gamma V_{th}$ or $\lVert \mathbf{W}^1\rVert_2\lVert \mathbf{F}^N\rVert_2\cdots\lVert \mathbf{F}^2\rVert_2\leq \gamma V_{th}^N$, where $\gamma<1$. To ensure the convergence of the forward SNN computation and stabilize training, we propose to restrict the spectral norm of the feedback connection weight matrix. Specifically, we re-parameterize $\mathbf{W}$ as:
\begin{equation}
    \mathbf{W} = \alpha \frac{\mathbf{W}}{\lVert \mathbf{W}\rVert_2},
\end{equation}
where $\alpha$ is a learnable parameter and will be clipped in the range of $[-c, c]$ ($c$ is a constant), and the spectral norm $\lVert \mathbf{W}\rVert_2$ is similarly computed as the implementation of Spectral Normalization~\cite{miyato2018spectral}. In experiments, we will set $V_{th}=2$ and $c=1$, and for the multi-layer structure, we only restrict the spectral norm of feedback connection weight $\mathbf{W}^1$. It works well in practice and the convergence is illustrated in Section 5.4 and Section \ref{appsec:convergence}.

\subsection{Batch Normalization}

Batch normalization (BN)~\cite{ioffe2015batch} is a commonly adopted technique in ANNs, which accelerates the training by reducing the internal covariate shift and improves performance as well. For a $d$-dimensional data $x=\left(x^{(1)}\cdots x^{(d)}\right)$, BN normalizes and transforms the data as:
\begin{equation}
    \hat{x}^{(k)} = \gamma^{(k)}\frac{x^{(k)}-\mathrm{E}[x^{(k)}]}{\sqrt{\mathrm{Var}[x^{(k)}]}}+\beta^{(k)},
\end{equation}
where $\mathrm{E}[x^{(k)}]$ and $\mathrm{Var}[x^{(k)}]$ are statistics over the training data set, and $\gamma^{(k)}, \beta^{(k)}$ are learnable parameters.

Note that when the statistics are fixed, BN is a simple linear transformation, and BN after a linear layer can be absorbed into the parameters of this layer. For example, for the linear operation $y=Wx+b$ (suppose $y$ is one-dimensional for simplicity), let $e, v$ denote the expectation and variance of $y$, then $\hat{y}=\text{BN}(y)$ is equivalent as a new linear operation $\hat{y}=\widetilde{W}x+\widetilde{b}$, where $\widetilde{W}=\frac{\gamma}{\sqrt{v}}W, \widetilde{b}=b-\frac{\gamma e}{v}+\beta$. Therefore, adding BN with fixed statistics after a convolution or fully-connected layer will not influence the properties of SNNs and the conclusions for equilibrium convergence.

We add BN after each linear operation except the feedback layer, in the context of the fixed-point equilibrium equation. For example, for the single-layer FSNN whose equation is $\mathbf{a^*} = \text{ReLU}\left(\frac{1}{V_{th}}\left(\mathbf{W}\mathbf{a^*}+\mathbf{F}\mathbf{x^*}+\mathbf{b}\right)\right)$, we add BN after $\mathbf{F}\mathbf{x^*}$; and for the multi-layer FSNN, we add BN after $\mathbf{F}^1\mathbf{x^*}$ and $\mathbf{F}^{l+1}\mathbf{a}^{l^*}$.

During forward SNN computation, the statistics of BN operations are fixed, i.e. we set BN into the 'eval' mode which uses the previously calculated statistics; and during backward gradient calculation, since it is decoupled from the forward computation (that means we will construct an additional computational graph for it), we can follow the common setting of BN to leverage the mini-batch estimated statistics and the overall statistics are updated, i.e. we set BN into the 'train' mode in this computational graph. The statistics are for the (weighted) average inputs or firing rates. Since the estimation of statistics for the forward SNN computation may be inaccurate in the first several iterations, we will use a warmup for the learning rate to alleviate this problem.

\subsection{Training Settings}

\subsubsection{Datasets}

We conduct experiments on MNIST~\cite{lecun1998gradient}, Fashion-MNIST~\cite{xiao2017fashion}, N-MNIST~\cite{orchard2015converting}, CIFAR-10 and CIFAR-100~\cite{krizhevsky2009learning}.

\paragraph{MNIST} MNIST is a dataset of handwritten digits with 10 classes, which is composed of 60,000 training samples and 10,000 testing samples. Each sample is a $28\times28$ grayscale image. We normalize the inputs based on the global mean and standard deviation, and convert the pixel value into a real-valued input current at every time step. No data augmentation is applied.

The licence of MNIST is the MIT License. The MNIST database is constructed from NIST's Special Database 3 and Special Database 1 which contain binary images of handwritten digits~\cite{lecun1998gradient}. The data does not contain personally identifiable information or offensive content since it only consists of handwritten digits.

\paragraph{Fashion-MNIST} Fashion-MNIST is a dataset similar to MNIST and contains $28\times28$ grayscale images of clothing items. We use the same preprocessing as MNIST.

The licence of Fashion-MNIST is the MIT License. The data of Fahion-MNIST is collected from the photographs of fashion products on the assortment on Zalando's website~\cite{lecun1998gradient}. The data does not contain personally identifiable information or offensive content since it only consists of 10 kinds of fashion products.

\paragraph{N-MNIST} N-MNIST is a neuromorphic dataset that is converted from MNIST by a Dynamic Version Sensor (DVS). It consists of spike trains triggered by the intensity change of pixels when DVS scans the static MNIST images along given directions. Since the intensity can either increase or decrease, there are two channels corresponding to ON- and OFF-event spikes. And the pixel dimension is expanded to $34\times34$ due to the relative shift of images. Therefore, each sample is a spike train pattern with the size of $34\times34\times2\times T$, where $T$ is the temporal length. The original data record $300ms$ with the resolution of $1\mu s$. We follow the prepossessing of \cite{zhang2020temporal} to reduce the time resolution by accumulating the spike train within every $3ms$, and we will use the first 30 time steps.

The license of N-MNIST is the Creative Commons Attribution-ShareAlike 4.0 license. The data is converted from MNIST and does not contain personally identifiable information or offensive content.

\paragraph{CIFAR-10} CIFAR-10 is a dataset of color images with 10 classes of objects, which is composed of 50,000 training samples and 10,000 testing samples. Each sample is a $32\times32\times3$ color image. We normalize the inputs based on the global mean and standard deviation, and apply random cropping and horizontal flipping for data augmentation. The input pixel value is converted to a real-valued input current at every time step as well.

\paragraph{CIFAR-100} CIFAR-100 is a dataset similar to CIFAR-10 except that there are 100 classes of objects. It also consists of 50,000 training samples and 10,000 testing samples. We use the same preprocessing as CIFAR-10.

The license of CIFAR-10 and CIFAR-100 is the MIT License. The data are labeled subsets of the 80 million tiny images datasets (collected from the web), which are labeled by students~\cite{krizhevsky2009learning}. The data does not contain personally identifiable information or offensive content, which is checked by the classes and image samples.

\subsubsection{Training Hyperparameters}

For all our SNN models, we set $V_{th}=2$. For the LIF model, we set $\lambda=0.95$ for MNIST, Fashion-MNIST and N-MNIST, while $\lambda=0.99$ for CIFAR-10 and CIFAR-100.

We train all our models by SGD with momentum for 100 epochs. We set the momentum as 0.9, the batch size as 128, and the initial learning rate as 0.05. For MNIST, Fashion-MNIST, and N-MNIST, the learning rate is decayed by 0.1 every 30 epochs, while for CIFAR-10 and CIFAR-100, it is decayed by 0.1 at the 50th and 75th epoch. We also apply linear warmup for the learning rate in the first 400 iterations for CIFAR-10 and CIFAR-100. We set the weight decay as $5\times 10^{-4}$, and apply the variational dropout as in \cite{bai2019deep,bai2020multiscale} with dropout rate as 0.2. For MNIST, Fashion-MNIST, CIFAR-10, and CIFAR-100, we solve the implicit differentiation by the Broyden's method proposed in \cite{bai2020multiscale} with the threshold as 30. For N-MNIST, we solve the implicit differentiation by the fixed-point update scheme indicated in Section 3.2 for 30 iterations, and the update scheme is modified as $x^T=\frac{1}{2}\left(x^T + (J_{f_{\theta}}^T\vert_{a^*})x^T+\left(\frac{\partial \mathcal{L}(a^*)}{\partial a^*}\right)^T\right)$ for acceleration. The initialization of parameters follows~\cite{wu2018spatio}, which first samples the weight parameters from the standard uniform distribution and then normalize them for each output dimension. All experiments are repeated five times and we report the mean, standard deviation, and the best results.

The code implementation is based on the PyTorch framework~\cite{paszke2019pytorch}, and experiments are carried out on one NVIDIA GeForce GTX 1080 GPU or one NVIDIA GeForce RTX 3090 GPU.

\section{Additional Experiment Results}

\subsection{Comparison between IF and LIF Model on CIFAR-10 and CIFAR-100}

In this subsection, we supplement the comparison results of IF and LIF model on CIFAR-10 and CIFAR-100, as shown in Table~\ref{cifar_supp}. It shows that the LIF model has similar performance compared with the IF model, and slightly outperforms the IF model in most cases, especially when the number of time steps is small. This also accords with the results on MNIST, Fashion-MNIST, and N-MNIST. The possible reason is that the LIF model leverages temporal information of spike trains by encoding weighted average firing rates. While each spike contributes equally to the average firing rate of the IF model and thus the precision of firing rates is only $\frac{1}{T}$, the weight for a spike of the LIF model is different at time steps (the weight is $\lambda^{T-t}$), and therefore the weighted average firing rates could encode more information with the same amount of time steps. When there is a relatively small number of time steps, the convergence errors of IF and LIF model would be similar and will not significantly affect the results. So the LIF model with temporal information may perform slightly better.

\begin{table} [ht]
	\centering
	\small
	\tabcolsep=1mm
	\caption{Comparison results of IF and LIF Model on CIFAR-10 and CIFAR-100.}
	\begin{center}
	\begin{threeparttable}
	\begin{tabular}{ccccc}
	    \multicolumn{5}{c}{\textbf{CIFAR-10}}\\
		\toprule[1pt]
		Network structure & Time steps & Model & Mean$\pm$Std & Best \\
		\midrule[0.5pt]
		\multirow{2}{*}{AlexNet-F} & \multirow{2}{*}{30} & IF & 91.73\%$\pm$0.13\% & 91.85\%\\
		& & LIF & 91.74\%$\pm$0.09\% & 91.92\%\\
		\midrule[0.5pt]
		\multirow{2}{*}{AlexNet-F} & \multirow{2}{*}{100} & IF & 92.25\%$\pm$0.27\% & 92.53\%\\
		& & LIF & 92.03\%$\pm$0.07\% & 92.15\%\\
		\midrule[0.5pt]
		\multirow{2}{*}{CIFARNet-F} & \multirow{2}{*}{30} & IF & 91.94\%$\pm$0.14\% & 92.12\%\\
		& & LIF & 92.08\%$\pm$0.15\% & 92.23\%\\
		\midrule[0.5pt]
		\multirow{2}{*}{CIFARNet-F} & \multirow{2}{*}{100} & IF & 92.33\%$\pm$0.15\% & 92.57\%\\
		& & LIF & 92.52\%$\pm$0.17\% & 92.82\%\\
		\bottomrule[1pt]
	\end{tabular}
	\begin{tablenotes}
       \scriptsize
       \item[1] AlexNet-F: 96C3s-256C3-384C3s-384C3-256C3 (F96C3u)
       \item[2] CIFARNet-F: 128C3s-256C3-512C3s-1024C3-512C3 (F128C3u)
    \end{tablenotes}
    \vspace{2mm}
	\end{threeparttable}
    \begin{tabular}{ccccc}
	    \multicolumn{5}{c}{\textbf{CIFAR-100}}\\
		\toprule[1pt]
		Network structure & Time steps & Model & Mean$\pm$Std & Best\\
		\midrule[0.5pt]
		\multirow{2}{*}{CIFARNet-F} & \multirow{2}{*}{30} & IF & 71.56\%$\pm$0.31\% & 72.10\%\\
		& & LIF & 71.72\%$\pm$0.22\% & 72.03\%\\
		\midrule[0.5pt]
		\multirow{2}{*}{CIFARNet-F} & \multirow{2}{*}{100} & IF & 73.07\%$\pm$0.21\% & 73.43\%\\
		& & LIF & 72.98\%$\pm$0.13\% & 73.12\%\\
		\bottomrule[1pt]
	\end{tabular}
	\end{center}
	\label{cifar_supp}
\end{table}

\begin{figure}
    \centering
    \subfigure[MNIST: 64C5s (F64C5)]{
    \includegraphics[scale=0.4]{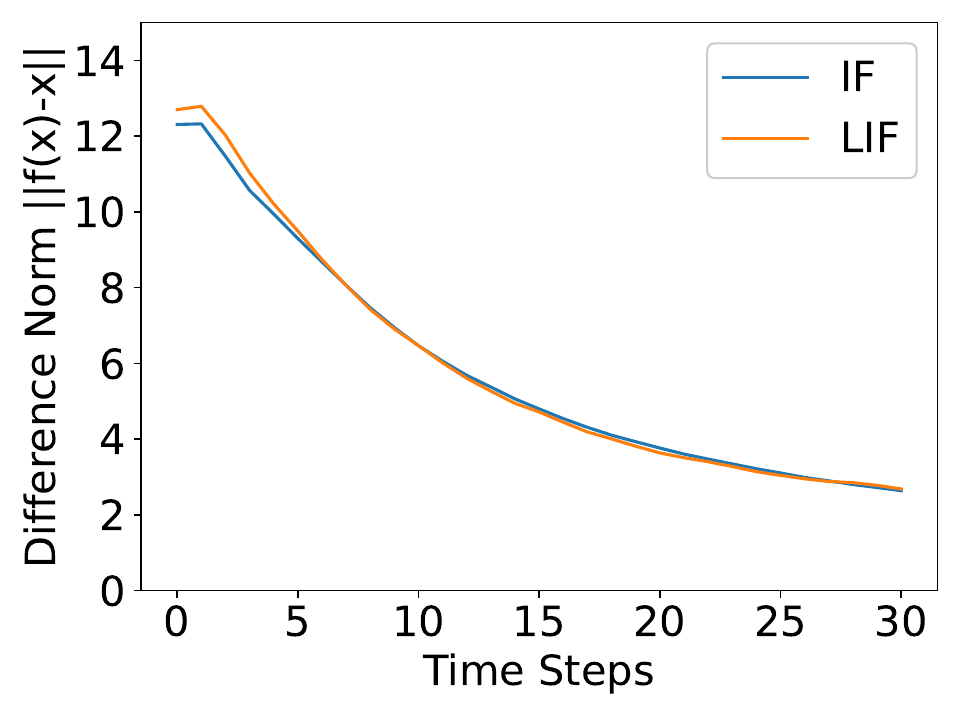}
    \label{convergence_mnistsup}
    }
    \subfigure[N-MNIST: 64C5s (F64C5)]{
    \includegraphics[scale=0.4]{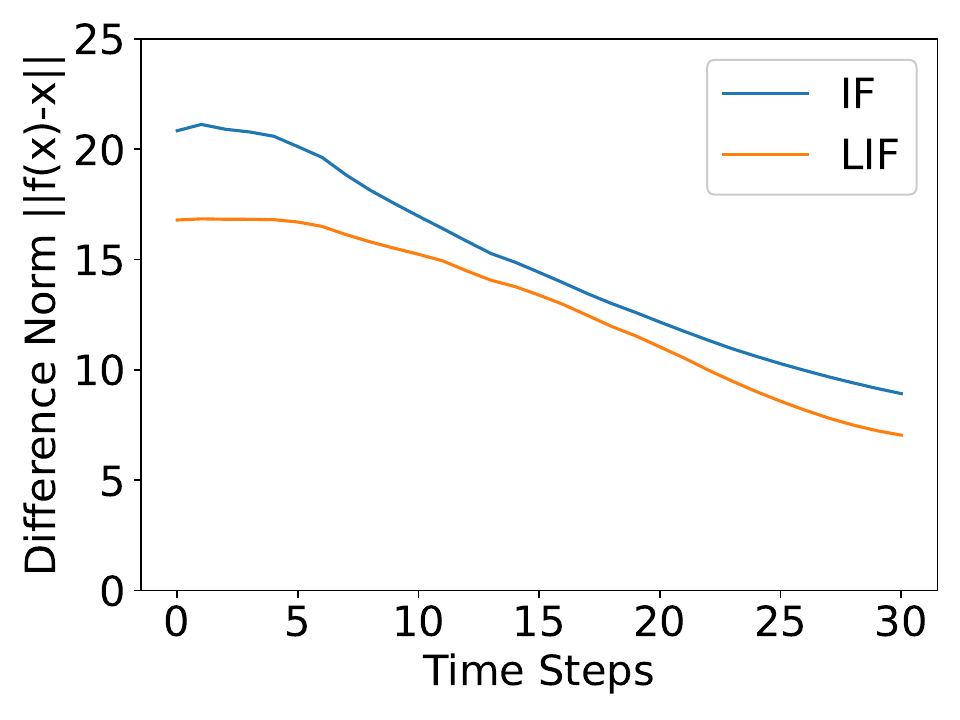}
    \label{convergence_nmnist}
    }
    \\
    \subfigure[CIFAR-10: CIFARNet-F]{
    \includegraphics[scale=0.4]{figures/cifar_large_all_100.pdf}
    \label{convergence_cifar10}
    }
    \subfigure[CIFAR-100: CIFARNet-F]{
    \includegraphics[scale=0.4]{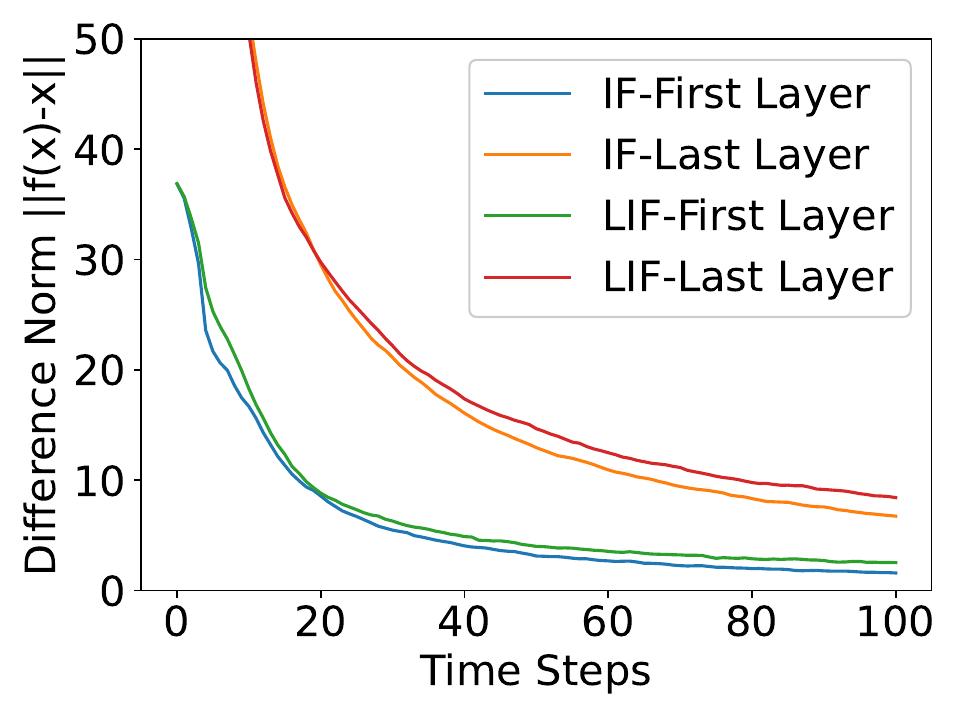}
    \label{convergence_cifar100}
    }
    \\
    \subfigure[CIFAR-10: CIFARNet-F]{
    \includegraphics[scale=0.4]{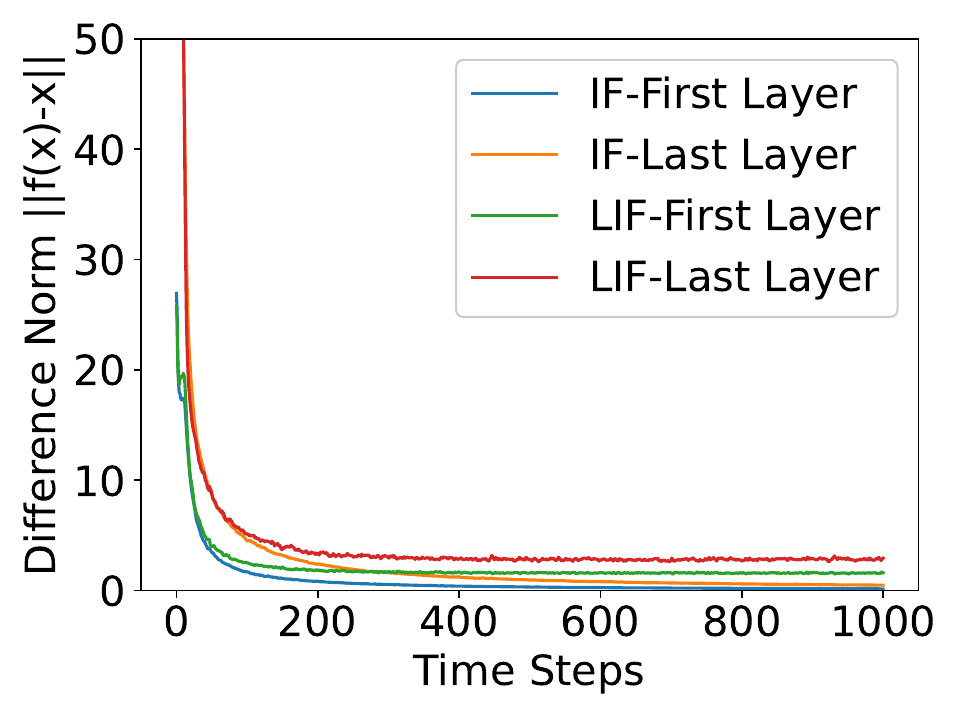}
    \label{convergence_cifar10_1000}
    }
    \subfigure[CIFAR-100: CIFARNet-F]{
    \includegraphics[scale=0.4]{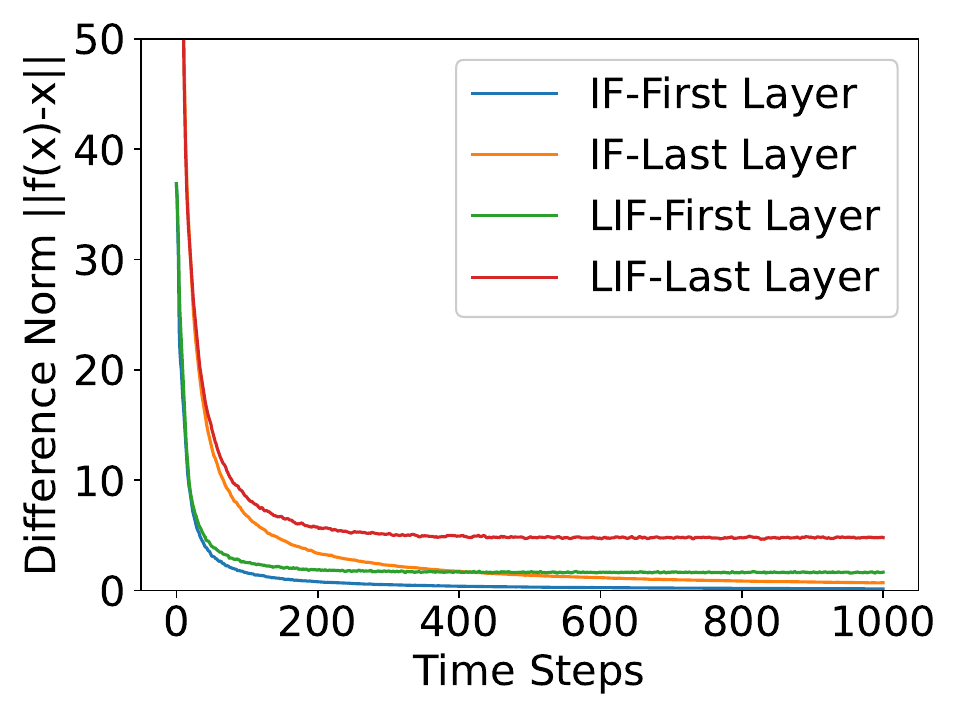}
    \label{convergence_cifar100_1000}
    }
    \caption{Convergence to the equilibrium of different models.}
    \label{fig:convergence_supp}
\end{figure}

\subsection{Convergence to Equilibrium}\label{appsec:convergence}

In this subsection, we supplement the results of convergence to the equilibrium state for more datasets and different scales of time steps. Figure~\ref{fig:convergence_supp} illustrates the convergence information on MNIST and N-MNIST with the same network structure under 30 time steps, as well as the convergence information on CIFAR-10 and CIFAR-100 with the same network structure under 100 and 1000 time steps. Since the precision of firing rates under relatively few time steps is limited, there would be errors caused by the precision. And the more the time steps are, the less the error should be. The difference norm decreases with time steps under all settings, demonstrating the convergence to the equilibrium state with the fixed-point equation. For N-MNIST, since the inputs are neuromorphic spikes rather than static images at each time step (and there lacks valid information in the first few time steps), the convergence is slower than on MNIST. Despite this, the (weighted) average firing rates do gradually approach the equilibrium as the difference norm decreases, and the training based on the implicit differentiation can work well as shown in the accuracy results. For the multi-layer structure, the convergence error of the last layer would be larger than the first layer. For the LIF model, since there would be random errors as indicated in the propositions, the convergence error would be larger than the IF model at most time. Nevertheless, when the number of time steps is small, the difference is not apparent. When the number of time steps comes to 1000, it shows that the error of IF model would continuously decrease, while the error of LIF model may stay in a bounded range, which should be caused by the random error. Despite these convergence errors, the accuracy results demonstrate that the exact precision is not necessary for effective training based on the implicit differentiation, and we can actually achieve satisfactory results with a small number of time steps.

\section{Influence of Time Steps}

In this subsection, we further study the influence of time steps, i.e. how good the convergence to equilibrium states needs to be for effective prediction and training. We first study the performance of a pretrained model under different time steps. The results are in Table~\ref{influence of time step pretrain}. It shows that the accuracy will gradually decrease as the time step decreases, and when the time step is 5, the classification totally fails. We also briefly analyze the total average firing rate under these conditions as in Table~\ref{influence of time step pretrain}, and it shows that the firing rate is very low when the time step is small. So the accuracy drop may also be partly due to the insufficient spikes.

\begin{table} [ht]
	\centering
	\small
	\tabcolsep=0.5mm
	\caption{Performance of a pretrained model under different time steps. The model is trained on CIFAR-10 with AlexNet-F structure and LIF model, and the original time step is 30.}
	\begin{tabular}{ccc}
		\toprule[1pt]
		Time steps & Accuracy & Total average firing rate\\
		\midrule[0.5pt]
		5 & 10.67\% & 0.0016\\
		10 & 74.39\% & 0.0046\\
		15 & 87.07\% & 0.0059\\
		20 & 90.16\% & 0.0063\\
		25 & 91.33\% & 0.0065\\
		30 & 91.82\% & 0.0066\\
		\bottomrule[1pt]
	\end{tabular}
	\label{influence of time step pretrain}
\end{table}

Then to study the influence of time steps on training, we train and test our model with only 5 time steps. The results are in Table~\ref{influence of time step training}. It shows that the training does not fail when the number of time steps is 5, but there would be a significant performance drop, and the accuracy would decrease and fluctuate in the latter part of training. It is probably because the gradient calculated by implicit differentiation could still be a descent direction though not exact, and in the latter part, it may not be a descent direction so the accuracy cannot be further improved.

\begin{table} [ht]
	\centering
	\small
	\tabcolsep=0.5mm
	\caption{Performance of training the model under different time steps. The model is trained on CIFAR-10 with AlexNet-F structure and LIF model.}
	\begin{tabular}{cc}
		\toprule[1pt]
		Time steps & Accuracy\\
		\midrule[0.5pt]
		5 & 83.09\%\\
		30 & 91.74\%$\pm$0.09\%\\
		\bottomrule[1pt]
	\end{tabular}
	\label{influence of time step training}
\end{table}

\section{Discussion of Limitations and Social Impacts}

This work mainly focuses on training feedback spiking neural networks for inputs that are convergent in the context of average accumulated signals, as indicated in the assumptions in the theorems. This holds for common pattern recognition tasks and common visual tasks, e.g. image classification, whose inputs are static images or the alternative neuromorphic version with spikes. While for other types of varying inputs, e.g. speech, it may require additional efforts to consider the definition and utilization of equilibrium with time. One practically plausible method is to flatten the inputs to treat the original time dimension as the channel dimension, and feed such data to the model at each `time step'. In this way, our theorems and method still hold. But the definition of  `time step' in this method is not the true time, which may lack the biological plausibility and increase the computational requirements. An interesting future work is to generalize the methodology to varying inputs.

As for social impacts, since this work focuses only on training methods for spiking neural networks, there is no direct negative social impact. And we believe that the development of successful energy-efficient SNN models could broader its applications and alleviate the huge energy consumption by ANNs. Besides, understanding and improving the training of biologically plausible SNNs may also contribute to the understanding of our brains and bridge the gap between biological neurons and successful deep learning.

\end{document}